\documentclass[a4paper]{article} 

\usepackage{graphicx} 
\usepackage{subfigure} 
\usepackage{fancyhdr}
\usepackage{fancybox}
\usepackage{hyperref}

\usepackage{amsmath,amsfonts,amssymb,amsthm}
\usepackage{array}
\usepackage{xspace}
\usepackage{graphicx}
\usepackage{color}
\usepackage{algorithmic, algorithm,verbatim}
\usepackage{enumerate}

\newtheorem{proposition}{Proposition}

\theoremstyle{remark}
\newtheorem{remark}{Remark}

\DeclareMathOperator*{\argmin}{argmin}
\DeclareMathOperator*{\argmax}{argmax}

\DeclareMathOperator{\prox}{prox}

\newcommand{\pierre}[1]{\textcolor{cyan}{{\bf !Pierre:} #1 {\bf :erreiP!}}}

\def\Cin{C_{\text{in}}}
\def\Cout{C_{\text{out}}}
\def\Cglob{C_{\text{glob}}}
\def\Mmax{M_{\text{max}}}

\def\bfn{\boldsymbol{n}}

\usepackage{exscale}
 

\makeatletter
\newenvironment{equationsize*}[1]{%
  \skip@=\baselineskip 
  #1%
  \baselineskip=\skip@ 
  \equation
}{\nonumber\endequation \ignorespacesafterend} 
\makeatother

 

\makeatletter
\newenvironment{alignsize*}[1]{%
  \skip@=\baselineskip 
  #1%
  \baselineskip=\skip@ 
  \start@align\@ne\st@rredtrue\m@ne
}{\endalign\ignorespacesafterend} 
\makeatother

\begin{document}
\lhead{P.Machart, S.Anthoine, L.Baldassarre} 
\rhead{Optimal Computational Trade-Off of Inexact Proximal Methods}
\rfoot[Technical Report V 1.0]{\thepage} 
\cfoot{} 
\lfoot[\thepage]{Technical Report V 2.0}

\renewcommand{\headrulewidth}{0.4pt}  
\renewcommand{\footrulewidth}{0.4pt}

\title{Optimal Computational Trade-Off of Inexact Proximal Methods}

\author{
Pierre Machart\\
LIF, LSIS, CNRS\\
Aix-Marseille University\\
\texttt{pierre.machart@lif.univ-mrs.fr} \\
\and
Sandrine Anthoine \\
LATP, CNRS, Aix-Marseille University \\
\texttt{anthoine@cmi.univ-mrs.fr} \\
\and
Luca Baldassarre \\
LIONS, \'Ecole Polytechnique F\'ed\'erale de Lausanne \\
\texttt{luca.baldassarre@epfl.ch} \\
}

\maketitle

\section*{Abstract}
In this paper, we investigate the trade-off between convergence rate and computational cost when minimizing a composite functional
 with proximal-gradient methods, which are popular optimisation tools in machine learning.
We consider the case when the \emph{proximity operator} is computed via an iterative procedure, which provides an approximation
of the exact proximity operator. In that case, we obtain algorithms with two nested loops.
We show that the strategy that minimizes the computational cost
 to reach a solution with a desired accuracy in finite time is to set the number of inner iterations to a constant,
 which differs from the strategy indicated by a convergence rate analysis.
In the process, we also present a new procedure called SIP (that is Speedy Inexact Proximal-gradient algorithm)
 that is both computationally efficient and easy to implement.
Our numerical experiments confirm the theoretical findings and suggest that SIP can be a very competitive alternative to the standard procedure.

\section{Introduction}


Recent advances in machine learning and signal processing have led to more involved optimisation problems,
 while abundance of data calls for more efficient optimization algorithms.
First-order methods are now extensively employed to tackle these issues and, among them,
 proximal-gradient algorithms~\cite{combettes2005signal,Nesterov07, Beck09} are becoming increasingly popular. 
They make it possible to solve very general convex non-smooth problems of the following form:
\begin{align}
\label{compprob}
 \min_x f(x) := g(x) + h(x),
\end{align}
where $g: \mathbb{R}^n \to \mathbb{R}$ is convex and smooth with an $L$-Lipschitz continuous gradient and $h: \mathbb{R}^n \to \mathbb{R}$
is lower semi-continuous proper convex, 
with remarkably simple, while effective, iterative algorithms
 which are guaranteed \cite{Beck09} to achieve the optimal convergence rate of $O(1/k^2)$, for a first order method, in the sense of \cite{nemirovsky1983}. 
They have been applied to a wide range of problems, from supervised learning with sparsity-inducing
 norm~\cite{bach2011convex,chen2011smoothing, baldassarre2012general, mosci2010solving},
 imaging problems~\cite{chambolle2004algorithm, beck2009tv, fadili2011total},
 matrix completion~\cite{cai2010singular, lin2009fast}, sparse coding~\cite{jenatton2011proximal} and multi-task learning~\cite{chen2009accelerated}.

The heart of these procedures is the \emph{proximity operator}.
In the favorable cases, analytical forms exist. However, there are many problems,
 such as Total Variation (TV) denoising and deblurring~\cite{chambolle2011first},
 non-linear variable selection~\cite{mosci2010solving}, structured sparsity~\cite{jenatton2011proximal,baldassarre2012general},
 trace norm minimisation \cite{cai2010singular, lin2009fast}, matrix factorisation problems such as
 the one described in~\cite{Schmidt11}, where the proximity operator can only be computed numerically, giving rise 
to what can be referred to as \emph{inexact proximal-gradient} algorithms~\cite{Schmidt11,villa2011aifobos}.

\begin{figure}[!b]
\centering
\includegraphics[width=0.6\textwidth]{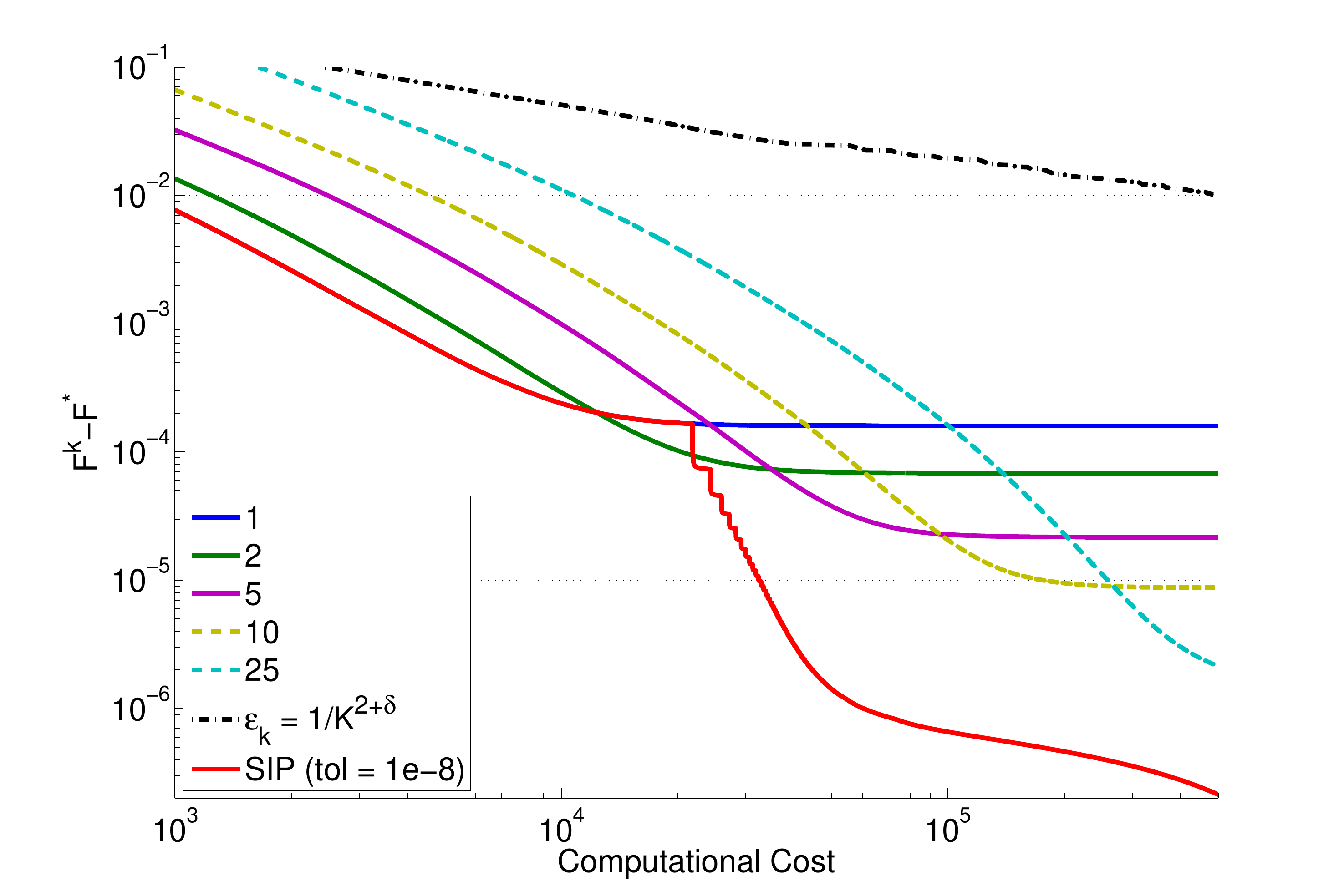} 
\caption{TV-regularization : Computational cost vs. objective value for different strategies}
\label{fig:intro}
\end{figure}

Both theory and experiments show that the precision of those numerical approximations 
has a fundamental impact on the performance of the algorithm.
A simple simulation, experimenting different strategies for setting this
 precision, on a classical Total Variation image deblurring problem (see Section \ref{proxtrad:expe} for more details)
 highlights two aspects of this impact.
Fig.~\ref{fig:intro} depicts the evolution of the objective value (hence precision)
 versus the computational cost (i.e. running time, see Section \ref{proxtrad:pb}
 for a more formal definition). The different curves are obtained by solving the exact same problem
of the form (\ref{compprob}), using, along the optimization process, either a constant
 precision (for different constant values) for the computation of the proximity operator, 
or an increasing precision (that is computing the proximity operator more and more precisely along the process).
 It shows that the computation cost required to reach a fixed value of the objective value 
varies greatly between the different curves (i.e. strategies).
That means that the computational performance of the procedure will dramatically depend on the chosen strategy.
Moreover, we can see that the curves do not reach the same plateaus, meaning that
the different strategies cause the algorithm to converge to different solutions, with different precisions.

Meanwhile, designing learning procedures which ensure good generalization properties is a central and recurring problem in Machine Learning.
When dealing with \emph{small-scale} problems, this issue is mainly covered by the traditional trade-off between \emph{approximation}
(i.e. considering the use of predictors that are complex enough to handle sophisticated prediction tasks) and \emph{estimation} (i.e. considering
 the use of predictors that are simple enough not to overfit on the training data).
However, when dealing with \emph{large-scale} problems, the amount of available data can make it impossible to precisely solve 
for 
the optimal trade-off between approximation and estimation.
Building on that observation, \cite{Bottou07} has highlighted that \emph{optimization} should be taken into account as a third crucial component,
in addition to approximation and estimation, leading to more complex (multiple) trade-offs.

In fact, dealing with the aforementioned multiple trade-off in a finite amount of computation time urges machine learners
 to consider solving problems with a lower precision and  pay closer attention to the  computational cost of 
the optimization procedures.
This crucial point motivates the study of strategies that lead to an approximate solution, at a smaller computational cost,
as the figure depicts.
However, the choice of the strategy that determines the precision of the numerical approximations
 seems to be often overlooked in practice.
Yet, in the light of what we have discussed in that introduction, we think it is pivotal.
In several studies, the precision is set so that the global algorithm 
converges to an optimum of the functional \cite{chaux09}, by studying sufficient conditions for such a convergence.
In many others, it is only considered as a mere implementation 
detail. A quick review of the literature shows that many application-centered papers
seem to neglect this aspect and fail at providing any detail regarding this point (e.g. \cite{anthoine12}).

Recently, some papers have addressed this question from a more theoretical point of view.
 For instance, \cite{Schmidt11,villa2011aifobos} give conditions on the approximations of the proximity
 operator so that the optimal convergence rate is still guaranteed.
However, rate analysis is not concerned by the complexity of computing the proximity operator.
As a consequence, the algorithms yielding the fastest rates of convergence are not necessarily the
computationally lightest ones, hence not the ones yielding the shortest computation time. 
In fact, no attempts have yet been made to assess the global computational cost of those inexact proximal-gradient algorithms.
 It is worth mentioning that for some specific cases, other types of proximal-gradient algorithms have
 been proposed that allow to avoid computing complex proximity operator \cite{loris2011generalization, chambolle2011first}.

In Section \ref{proxtrad:setting}, we start from the results in~\cite{Schmidt11} that link the overall accuracy
of the iterates of inexact proximal-gradient methods with
 the errors in the approximations of the proximity operator. 
We consider iterative methods for computing the proximity operator and, in Section \ref{proxtrad:pb}, we show that if one is
 interested in minimizing the computational cost (defined in Section~\ref{subsec:tradeoff})
 for achieving a desired accuracy, other strategies than the ones proposed in ~\cite{Schmidt11}
 and~\cite{villa2011aifobos} might lead to significant computational savings.

The main contribution of our work is showing, in Section~\ref{proxtrad:res}, that for both accelerated and non-accelerated proximal-gradient methods,
 the strategy minimizing the global cost to achieve a desired accuracy is to keep the number of internal iterations constant.
 This constant depends on the desired accuracy and the convergence rate of the algorithm used to compute the proximity operator.
Coincidentally, those theoretical strategies meet those of actual implementations and widely-used packages
 and help us understand both their efficiency and limitations.
After a discussion on the applicability of those strategies, we also propose a more practical one,
 namely the Speedy Inexact Proximal-gradient
 (SIP) strategy, motivated by our analysis.

In Section \ref{proxtrad:expe}, we numerically assess different strategies (i.e. constant numbers of inner iterations, SIP, 
the strategy yielding optimal convergence rates)
 on two problems, illustrating the theoretical analysis and suggesting that our new strategy SIP can be very effective.
This leads to a final discussion about the relevance and potential limits of our approach along with some hints
 on how to overcome them.

\section{Setting}
\label{proxtrad:setting}
\subsection{Inexact Proximal Methods}

To solve problem (\ref{compprob}), one may use the so-called \emph{proximal}-gradient methods~\cite{Nesterov07}.
Those iterative methods consist in generating a sequence $\{x_k\}$, where
$$x_k = \prox_{h/L} \left[y_{k-1} - \tfrac{1}{L} \nabla g(y_{k-1})\right],$$
$$\text{with } y_k = x_k + \beta_k (x_k - x_{k-1}),$$
and the choice of $\beta_k$ gives rise to two schemes: $\beta_k = 0$ for the \emph{basic scheme},
or some well-chosen sequence (see \cite{Nesterov07,Tseng08,Beck09} for instance) for an \emph{accelerated scheme}.
The {\em proximity operator} $\prox_{h/L}$ is defined as:
\begin{equation}
\label{defprox}
 \prox_{h/L}(z) = \argmin_x \tfrac{L}{2} \|x-z\|^2 + h(x).
\end{equation}

In the most classical setting, the proximity operator is computed exactly. 
The sequence $\{x_k\}$ then converges to the solution of problem (\ref{compprob}).
However, in many situations 
no closed-form solution of (\ref{defprox}) is known and one can only provide an approximation of the proximal point.
From now on, let us denote by $\epsilon_k$ an upper bound on the error induced in the proximal
 objective function by this approximation, at the $k$-th iteration:
\begin{align}
\label{proxerr}
\frac{L}{2} \|x_k-z\|^2 + h(x_k) \leq \epsilon_k + \min_x \left\{\frac{L}{2} \|x-z\|^2 + h(x)\right\}.
\end{align}

For the basic scheme, the convergence of $\{x_k\}$ to the optimum of
 Problem (\ref{compprob}) has been studied in~\cite{combettes2005signal} and
is ensured under fairly mild conditions on the sequence $\{\epsilon_k\}$.

\subsection{Convergence Rates}

The authors of~\cite{Schmidt11} go beyond the study on the convergence of inexact proximal methods: they establish their rates of convergence.
(This is actually done in the more general case where the gradient of $g$ is also approximated. In the present study, we restrict ourselves to error in the proximal part.)

Let us denote by $x^*$ the solution of problem (\ref{compprob}). The convergence rates of the basic (non-accelerated) proximal method (e.g. $y_k = x_k$) thus reads:
\begin{proposition}[Basic proximal-gradient method (Proposition 1 in~\cite{Schmidt11})]
 For all $k \geq 1$,
\begin{align}
\label{convrateconv}
 f(x_k) - f(x^*) \leq \frac{L}{2k}\left(\|x_0 - x^*\| 
+ 2 \sum_{i=1}^k \sqrt{\frac{2 \epsilon_i}{L}} + \sqrt{\sum_{i=1}^k \frac{2 \epsilon_i}{L}}\right)^2.
\end{align}
\end{proposition}

\begin{remark}
 In~\cite{Schmidt11}, this bound actually holds on the average of the iterates $x_i$, i.e.
 $$f\left(\frac{1}{k}\sum_{i=1}^k x_i\right) - f(x^*)\leq \frac{L}{2k}\left(\|x_0 - x^*\| 
+ 2 \sum_{i=1}^k \sqrt{\frac{2 \epsilon_i}{L}} + \sqrt{\sum_{i=1}^k \frac{2 \epsilon_i}{L}}\right)^2.$$
 (\ref{convrateconv}) thus holds for the iterate that achieve the lowest function value. It also trivially holds all the time for algorithms with which
the objective is non-increasing.
\end{remark}

The convergence rate of accelerated schemes (e.g. $y_k = x_k + \frac{k-1}{k+2} x_{k-1}$) reads:
\begin{proposition}[Accelerated proximal-gradient method (Proposition 2 in~\cite{Schmidt11})]
 For all $k \geq 1$, 
\begin{align}
 \label{convrateconvacc}
 f\left(x_k\right) - f(x^*) \leq \frac{2L}{(k+1)^2}\left(\|x_0 - x^*\| 
+ 2 \sum_{i=1}^k  i \sqrt{\frac{2 \epsilon_i}{L}} + \sqrt{\sum_{i=1}^k \frac{2 i^2 \epsilon_i}{L}}\right)^2.
\end{align}
\end{proposition}

Bounds with faster rates (Proposition 3 and 4 in~\cite{Schmidt11}) can be obtained if the objective is strongly convex.
Some results will be briefly mentioned in Section~\ref{proxtrad:res} in this case.
However, we will not detail them as much as in the more general setting.

\subsection{Approximation Trade-off}
The inexactitude in the computation of the proximity operator imposes two additional terms in each bound, for instance in (\ref{convrateconv}): :
$$2 \sum_{i=1}^k \sqrt{\frac{2 \epsilon_i}{L}} \text{ and } \sqrt{\sum_{i=1}^k \frac{2 \epsilon_i}{L}}.$$
When the $\epsilon_i$'s are set to $0$ (i.e. the proximity operator is computed exacted), one obtains the usual 
bounds of the exact proximal methods.
These additional terms (in (\ref{convrateconv}) and (\ref{convrateconvacc}) resp.)
are summable if $\{\epsilon_k\}$ converges at least as fast as $O\left(\frac{1}{k^{(2+\delta)}}\right)$
 (resp. $O\left(\frac{1}{k^{(4+\delta)}}\right)$), for any $\delta>0$.
One direct consequence of these bounds (in the basic and accelerated schemes respectively)
 is that the optimal convergence rates in the error-free setting are still achievable, with such conditions on the $\{\epsilon_k\}$'s.
Improving the convergence rate of $\{\epsilon_k\}$ further causes the additional terms to sum to smaller constants, hence
inducing a faster convergence of the algorithm without improving the rate.
However,~\cite{Schmidt11} empirically notices that 
 imposing too fast a decrease rate on $\{\epsilon_k\}$ is computationally counter-productive, as the precision required
 on the proximal approximation becomes computationally demanding.
In other words, there is a subtle trade-off between the number of iterations needed to reach a certain solution
 and the cost of those iterations. This is
 the object of study of the present paper.

\section{Defining the Problem}
\label{proxtrad:pb}

The main contribution of this paper is to define a \emph{computationally optimal} way of
 setting the trade-off between the number of iterations and their cost, in various situations.
We consider the case where the proximity operator is approximated \emph{via} an iterative procedure.
The global algorithm thus consists in an iterative proximal method, where at each (outer-)iteration, one performs (inner-)iterations.

With that setting, it is possible to define (hence optimize) the global computational cost of the algorithm.
If the convergence rate of the procedure used in the inner-loops is known,
the main result of this study provides a strategy to set the number of inner iterations
 that minimizes the cost of the algorithm, under some constraint upper-bounding 
the precision of the solution (as defined in (\ref{genprob})).

\subsection{The Computational Cost of Inexact Proximal Methods}
As stated earlier, our goal is to take into account the complexity of the global cost of inexact proximal methods.
Using iterative procedures to estimate the proximity operator at each step, it is possible to formally express this cost.
Let us assume that each inner-iteration has a constant computational cost $\Cin$ 
 and that, in addition to the cost induced by the inner-iterations, each outer-iteration has a constant computational cost $\Cout$.
It immediately follows that the global cost of the algorithm is:
\begin{equation}
\Cglob(k,\{l_i\}_{i=1}^k) = \Cin \sum_{i=1}^k l_i + k \Cout,
\label{eq:cost}
\end{equation}
and the question we are interested in is to minimize this cost.
In order to formulate our problem as a minimization of this cost, subject to some guarantees on the global precision of the solution, 
we now need to relate the number of inner iterations to the precision of the proximal point estimation.
This issue is addressed in the following subsections.

\subsection{Parameterizing the Error}
Classical methods to approximate the proximity operator
achieve either \emph{sublinear rates} of the form $O\left(\frac{1}{k^\alpha}\right)$  ($\alpha = \tfrac{1}{2}$ for sub-gradient or stochastic gradient descent in the general case;
 $\alpha = 1$ for gradient and proximal descent or $\alpha = 2$ for accelerated descent/proximal schemes) or \emph{linear rates} $O\left((1-\gamma)^k\right)$ (for strongly convex objectives or second-order methods).
Let $l_i$ denote the number of inner iterations performed at the $i$-th iteration of the outer-loop.
We thus consider two types of upper bounds on the error defined in (\ref{proxerr}):
\begin{align}
\label{epsilonbound}
\epsilon_i = \frac{A_i}{l_i^\alpha}\quad\text{(sublinear rate)} \qquad\text{or}\qquad \epsilon_i = A_i(1-\gamma)^{l_i}\quad\text{(linear rate)},
\end{align}
for some positive $A_i$'s.

\subsection{Parameterized Bounds}
\label{subsec:bounds}
Plugging (\ref{epsilonbound}) into (\ref{convrateconv}) or (\ref{convrateconvacc}), we can get four different global bounds:
\begin{align*}
 f\left(x_k\right) - f(x^*) \leq B_j(k,\{l_i\}_{i=1}^k),\ j=1,..,4,
\end{align*}
depending on whether we are using a basic or accelerated scheme on the one hand, and on whether we have sub-linear or
 linear convergence rate in the inner-loops on the other hand.
More precisely, we have the following four cases:
\begin{enumerate}
 \item basic out, sub-linear in:
  $$B_1(k,\{l_i\}_{i=1}^k) = \frac{L}{2k}\left(\|x_0 - x^*\| + 3 \sum_{i=1}^k \sqrt{\frac{2 A_i}{L l_i^\alpha}}\right)^2$$
 \item basic out, linear in:
  $$B_2(k,\{l_i\}_{i=1}^k) = \frac{L}{2k}\left(\|x_0 - x^*\| + 3 \sum_{i=1}^k \sqrt{\frac{2 A_i (1-\gamma)^{l_i}}{L}}\right)^2$$
 \item accelerated out, sub-linear in:
  $$B_3(k,\{l_i\}_{i=1}^k) = \frac{2 L}{(k+1)^2}\left(\|x_0 - x^*\| + 3 \sum_{i=1}^k i \sqrt{\frac{2 A_i}{L l_i^\alpha}}\right)^2$$
 \item accelerated out, linear in:
  $$B_4(k,\{l_i\}_{i=1}^k) = \frac{2 L}{(k+1)^2}\left(\|x_0 - x^*\| + 3 \sum_{i=1}^k i \sqrt{\frac{2 A_i (1-\gamma)^{l_i}}{L}}\right)^2$$
\end{enumerate}

\subsection{Towards a Computationally Optimal Tradeoff}
\label{subsec:tradeoff}
Those bounds highlight the aforementioned trade-off.
To achieve some fixed global error
$$\rho = f(x_k) - f(x^*)$$
, there is a natural trade-off that need to be set by the user, between the number $k$ of outer-iterations
 and the numbers of inner-iterations $\{l_i\}_{i=1}^k$, which can be seen as hyper-parameters of the global algorithms.
As mentioned earlier, and witnessed in~\cite{Schmidt11} the choice of those parameters will
 have a crucial impact on the computational efficiency (see equation (\ref{eq:cost})) of the algorithm.


Our aim to ``optimally'' set the hyper-parameters ($k$ and $\{l_i\}_{i=1}^k$) may be conveyed by the following optimization problem.
For some fixed accuracy $\rho$, we want to minimize the global cost $\Cglob$ of the algorithm,
 under the constraint that our bound on the error $B$ is smaller than $\rho$:
\begin{align}
\label{genprob}
 \min_{k \in \mathbb{N},\{l_i\}_{i=1}^k \in \mathbb{N}^k} &\Cin \sum_{i=1}^k l_i + k \Cout \qquad \text{s.t. } B(k,\{l_i\}_{i=1}^k) \leq \rho.
\end{align}
This optimization problem the rest of this paper will rest upon.

\section{Results}
\label{proxtrad:res}
Problem~(\ref{genprob}) is an integer optimization problem as the variables of interest are numbers of (inner and outer) iterations.
As such, this is a complex (NP-hard) problem and
 one cannot find a closed form for the integer solution, but if we relax our problem in $l_i$
 to a continuous one:
\begin{align}
\label{genprobcont}
 \min_{k \in \mathbb{N},\{l_i\}_{i=1}^k \in [1,\infty)^k} &\Cin \sum_{i=1}^k l_i + k \Cout \qquad \text{s.t. } B(k,\{l_i\}_{i=1}^k) \leq \rho,
\end{align}
 it actually is possible to find an \emph{analytic expression} of the optimal $\{l_i\}_{i=1}^k$ and
 to numerically find the optimal $k$.

\subsection{Optimal Strategies}
The next four propositions describe the solution of the relaxed version (\ref{genprobcont}) of Problem~(\ref{genprob}) in the four different scenarios defined in Section~\ref{subsec:bounds} and for a constant value $A_i=A$.

\subsubsection*{Scenarios 1 and 2: basic out}
Let $$C(k) = \frac{\sqrt{L}}{3\sqrt{2 A}}\left(\sqrt{\frac{2 k \rho}{L}} - \|x_0 - x^*\|\right).$$
Solving the continuous relaxation of problem (\ref{genprob}) with the bounds $B_1$ and $B_2$ leads to the following propositions:
\begin{proposition}[Basic out, sub-linear in]
\label{bosi}
 If $\rho < 6 \sqrt{2 L A} \|x_0 - x^*\|$,
 the solution of problem (\ref{genprob}) for $B = B_1$ is:
\begin{align}
 \label{kbosi}
\forall~i, l_i^*=\left(\frac{C(k^*)}{k^*}\right)^{-\frac{2}{\alpha}}, \text{ with }k^* = \argmin_{k \in \mathbb{N}^*} k \Cin \Big(\frac{C(k)}{k}\Big)^{-\frac{2}{\alpha}} + k \Cout.
\end{align}
\end{proposition}

\begin{proposition}[Basic out, linear in]
\label{boli}
 If $\rho < 6 \sqrt{2 L A (1-\gamma)} \|x_0 - x^*\|$,
 the solution of problem (\ref{genprob}) for $B = B_2$ is:
\begin{align}
 \label{kboli}
\forall~i, l_i^*=\frac{2 \ln\frac{C(k^*)}{k^*}}{\ln(1-\gamma)}, \text{ with }k^* = \argmin_{k \in \mathbb{N}^*}  \frac{2 k \Cin}{\ln(1-\gamma)} \ln \Big(\frac{C(k)}{k}\Big) + k \Cout.
\end{align}
\end{proposition}

\subsubsection*{Scenarios 3 and 4: accelerated out}
Let $$D(k) = \frac{\sqrt{L}}{3\sqrt{2 A}}\left(\sqrt{\frac{\rho}{ 2L}} (k+1) - \|x_0 - x^*\|\right).$$ 
Solving the continuous relaxation (\ref{genprobcont}) of problem (\ref{genprob}) with the bound $B_3$ leads to the following proposition:
\begin{proposition}[Accelerated out, sub-linear in]
\label{aosi}
 If $\rho < \left(\sqrt{12 \sqrt{2 L A}\|x_0 - x^*\|} - 3\sqrt{A}\right)^2$,
 the solution of problem (\ref{genprob}) for $B = B_3$ is:
\begin{align}
 \label{kaosi}
\forall i, l_i^*=\left(\frac{2 D(k^*)}{k^*(k^*+1)}\right)^{-\frac{2}{\alpha}}, \text{ with }k^* = \argmin_{k \in \mathbb{N}^*} k \Cin \Big(\frac{2 D(k)}{k(k+1)}\Big)^{-\frac{2}{\alpha}} + k\Cout.
\end{align}
\end{proposition}
A similar result holds for the last scenario: $B=B_4$. However in this case, the optimal $l_i$ are equal to $1$ up to $n(k^*)$ ($1 \leq n(k^*) < k^*$) and then increase with $i$:

\begin{proposition}[Accelerated out, linear in]
\label{aoli}
 If $\rho < \left(\sqrt{12 \sqrt{2 L A (1-\gamma)}\|x_0 - x^*\|} - 3\sqrt{A}\right)^2$,
 the solution of problem (6) for $B = B_4$ is:
$$l_i^* =\left\{\begin{array}{cl}
1 &\text{ for } 1\leq i \leq n(k^*)-1\\
\frac{2}{\ln(1-\gamma)} \left(\ln\bigg(\tfrac{D(k)-\tfrac{n(k)(n(k)-1)}{2}\sqrt{1-\gamma}}{k+1-n(k)}\bigg)\right) &\text{ for } n(k^*) \leq i \leq k^* \end{array} \right.$$

\begin{align}
 \label{kaoli}\nonumber
\text{with } k^* = \argmin_{k \in \mathbb{N}^*} \biggl\{k \Cout &+ \Cin (n(k)-1) - \tfrac{2 \Cin}{ \ln(1-\gamma)}\ln\left(\tfrac{k!}{n(k)!}\right)\\
&- \tfrac{2 \Cin(k-n(k)+1)}{ \ln(1-\gamma)} \ln\left(\tfrac{k+1-n(k)}{D(k)-\tfrac{n(k)(n(k)-1)}{2}\sqrt{1-\gamma}}\right)\biggr\},
\end{align}

and  $n(k)$ is defined as the only integer such that:
$${\big(n(k)-1\big)\big(2k+2-n(k)\big)} \sqrt{1-\gamma} \leq 2 D(k) <  {n(k)\big(2k+1-n(k)\big)} \sqrt{1-\gamma}.$$
\end{proposition}

\begin{proof}[Sketch of proof (For a complete proof, please see appendix~\ref{proxtrad:app}.)]
 First note that:
$$\min_{k,\{l_i\}_{i=1}^k} \Cin \sum_{i=1}^k l_i + k \Cout = \min_{k}\min_{\{l_i\}_{i=1}^k} \Cin \sum_{i=1}^k l_i + k \Cout.$$
We can  solve problem \eqref{genprob} by first solving, for any $k$, the minimization problem over $\{l_i\}_{i=1}^k$.
This is done using the standard Karush-Kuhn-Tucker approach~\cite{KKT}.
Plugging the analytic expression of those optimal $\{l_i^*\}_{i=1}^k$ into our functional, we get our problem in $k$.
\end{proof}

\begin{remark}
Notice that the propositions hold for $\rho$ smaller than a  threshold. If not, the analysis and results are different.
Since we focus on a high accuracy, we here develop the results for small values of $\rho$.
We defer the results for the larger values of $\rho$ to appendix~\ref{proxtrad:app}.
\end{remark}

\begin{remark}
\label{rem:kint}
In none of the scenarios can we provide an analytical expression of $k^*$.
However, the expressions given in the propositions allow us to exactly retrieve the solution.
The functions of $k$ to minimize are monotonically decreasing then increasing.
As a consequence, it is possible to numerically find the minimizer in $\mathbb{R}$, for instance in the first scenario:
$$\hat{k} = \argmin_{k \in \mathbb{R}} k \Cin \Big(\frac{C(k)}{k}\Big)^{-\frac{2}{\alpha}} + k \Cout,$$
with an arbitrarily high precision, using for instance a First-Order Method.
It follows that the integer solution $k^*$ is exactly either the flooring or ceiling of $\hat{k}$.
Evaluating the objective for the two possible roundings gives the solution.
\end{remark}

Finally, as briefly mentioned, bounds with faster rates can be obtained when the objective is known to be strongly convex.
In that case, regardless of the use of basic or accelerated schemes and of sub-linear or linear rates in the inner loops,
 the analysis leads to results similar to those reported in Proposition~\ref{aoli} (e.g. using $1$ inner iteration for the first rounds
 and an increasing number then). Due to the lack of usability and interpretability of these results, we will not report them here. 

\subsection{Comments and Interpretation of the Results}
\subsubsection*{Constant number of inner iterations}
Our theoretical results urge to use a constant number of inner iterations in 3 scenarios.
Coincidentally, many actual efficient implementations of such two nested algorithms, in~\cite{anthoine12} or in packages like
 SLEP\footnote{http://www.public.asu.edu/~jye02/Software/SLEP/index.htm} or PQN\footnote{http://www.di.ens.fr/~mschmidt/Software/PQN.html},
 use these constant number schemes.
However, the theoretical grounds for such an implementation choice were not explicited.
Our results can give some deeper understanding on why and how those practical implementations perform well.
They also help acknowledging that the computation gain comes at the cost of an intrinsic limitation to the precision of the obtained solution.

\subsubsection*{An Integer Optimization Problem}

The impact of the continuous relaxation of the problem in $\{l_i\}_{i=1}^{k^*}$
 is subtle. In practice, we  need to set the constant number on inner iterations $l_i$ to
 an integer number. Setting, $\forall i\in[1,k^*], l_i = \lceil l_i^*\rceil$ ensures that the final error is smaller than $\rho$.
This provides us with an approximate (but feasible) solution to the integer problem.

One may want to refine this solution by sequentially setting $l_i$ to $\lfloor l_i^*\rfloor$ (hence
reducing the computational cost), starting from $i=1$,
 while the constraint is met, i.e. the final error remains smaller than $\rho$.
Refer to Algorithm \ref{algoround} for an algorithmic description of the procedure.
\begin{algorithm}
\caption{A finer grain procedure to obtain an integer solution for the $l_i$'s}
\label{algoround}
\begin{algorithmic}
 \REQUIRE $\{l_i^*\}_{i=1}^{k^*}$
 \STATE $\forall i\in[1,k^*], l_i \leftarrow \lceil l_i^* \rceil$
 \STATE $i \leftarrow 1$
 \REPEAT
   \STATE $l_i \leftarrow \lfloor l_i^*\rfloor$
   \STATE $i \leftarrow i+1$
 \UNTIL $B(k^*,\{l_i\}_{i=1}^{k^*}) > \rho$
\end{algorithmic}
\end{algorithm}

\subsubsection*{Computationally-Optimal vs. Optimal Convergence Rates Strategies}
The original motivation of this study is to show how, in the inexact proximal methods setting, optimization strategies 
that are the most computationally efficient, given some desired accuracy $\rho$,
 are \emph{fundamentally different} from those that achieve optimal convergence rates.
The following discussion motivates why minding this gap is of great interest for machine learners
while an analysis of the main results of this work highlights it.

When one wants to obtain a solution with an arbitrarily high precision, optimal rate methods are of great interest:
regardless of the constants in the bounds, there always exists a (very high) precision $\rho$ beyond which methods with optimal rates will be faster
 than methods with suboptimal convergence rates.
However, when dealing with real large-scale problems, reaching those levels of precision
 is not computationally realistic. 
When taking into account budget constraints on the computation time,
 and as suggested by \cite{Bottou07}, generalization properties of the learnt function will depend on both 
statistical and computational properties.

At the levels of precision intrinsically imposed by the budget constraints,
taking other elements than the convergence rates becomes crucial for designing efficient procedures as our study shows.
Other examples of that phenomenon have been witnessed, for instance, when using Robbins-Monro algorithm (Stochastic Gradient Descent).
It has been long known (see \cite{polyak92} for instance) that that the use of a step-size proportional to the inverse of the number of iterations allows to reach the optimal
 convergence rates (namely $1/k$) .

On the other hand, using a non-asymptotic analysis \cite{bach11b}, one can prove (and 
observe in practice) that such a strategy can also lead 
to catastrophic results when $k$ is small (i.e. possibly a large increase of the objective value) and undermines the 
computational efficiency of the whole procedure.

Back to our study, for the first three scenarios (Propositions \ref{bosi}, \ref{boli} and \ref{aosi}), the computationally-optimal strategy imposes constant 
number of inner iterations. Given our parameterization, Eq. \eqref{epsilonbound}, this
 also means that the errors $\epsilon_i$ on the proximal computation remains constant.
On the opposite, the optimal convergence rates can only be achieved for sequences of $\epsilon_i$
 decreasing strictly faster than ${1}/{i^2}$ for the basic schemes and ${1}/{i^4}$ for the accelerated schemes.
Obviously, the optimal convergence rates strategies also yield a bound on the minimal number of outer iterations
needed to reach precision $\rho$ by inverting the bounds \eqref{convrateconv} or
 \eqref{convrateconvacc}.
However, this strategy is provably less efficient (computationally-wise) than the optimal one we have derived.

In fact, the pivotal difference between ``optimal convergence rates'' and ``computationally optimal'' strategies lies in
 the fact that the former ones arise from an asymptotic analysis while the latter arise from a finite-time analysis.
While the former ensures that the optimization procedure will converge to the optimum of the problem
 (with optimal rates in the worst case), the latter only ensures that after $k^*$ iterations, the solution found 
by the algorithm is not further than $\rho$ from the optimum.

\subsubsection*{Do not \emph{optimize further}}
To highlight this decisive point in our context, let us fix some arbitrary precision $\rho$.
Propositions \ref{bosi} to \ref{aosi} give
 us the optimal values $k^*$ and $\{l_i^*\}_{i=1}^{k^*}$ depending on the inner and outer algorithms we use.
Now, if one wanted to \emph{further optimize} by continuing the same strategy for $k'>k^*$ iterations (i.e. still running $l_i^*$ inner iterations), 
we would have the following bound: $$B(k',\{l_i^*\}_{i=1}^{k'}) > B(k^*,\{l_i^*\}_{i=1}^{k^*}) = \rho.$$
In other words, if one runs more than $k^*$ iterations of our optimal strategy, with the same $l_i$, we can not guarantee that the error still decreases.
In a nutshell, our strategy is precisely computationally optimal because it does not ensure more than what we ask for.

\subsection{On the Usability of the Optimal Strategies}
Designing computationally efficient algorithms or optimization strategies is motivated by practical considerations.
The strategies we proposed are provably the best to ensure a desired precision.
Yet, in a setting that covers a very broad range of problems, their usability can be compromised.
We point out those limitations and propose a solution to overcome them.

First, these strategies require the desired (absolute) precision to be known.
In most situations, it is actually difficult, if not impossible, to know in advance which precision will ensure
 that the solution found has desired properties (e.g. reaching some specific SNR ratio for image deblurring).
More critically, if it turned out that the user-defined precision was not sufficient, we showed that ``optimizing further''
 with the same number of inner iterations does not guarantee to improve the solution. For a sharper precision, one would
 technically have to compute the new optimal strategy and run it all over again.

Although it is numerically possible, evaluating the optimal number of iterations $k^*$ still requires to solve 
an optimization problem.
More importantly, the optimal values for the numbers of inner and outer iterations depend on quantities like $\|x_0 -x^*\|$  which are
 unknown and very difficult to estimate.
Those remarks undermine the direct use of the presented computationally optimal strategies.

To overcome these problems, we propose a new strategy called \emph{Speedy Inexact Proximal-gradient algorithm} (\emph{SIP}),
 described in Algorithm~\ref{newstrat}, which is motivated by our theoretical study and very simple to implement.
In a nutshell, it starts using only one inner iteration.
When the outer objective stops decreasing fast enough, the algorithm increases the number of internal iterations
 used for computing the subsequent proximal steps, until the objective starts decreasing fast enough again.

\begin{algorithm}
 \caption{Speedy Inexact Proximal-gradient strategy (\emph{SIP})}
 \label{newstrat}
\begin{algorithmic}
 \REQUIRE An initial point $x_0$, an update rule $\mathcal{A}_\text{out}$,
 an iterative algorithm $\mathcal{A}_\text{in}$ for computing the proximity operator, a tolerance $\text{tol} > 0$, a stopping criterion $\text{STOP}$.
\STATE $x \leftarrow x_0$, $l \leftarrow 1$

  \REPEAT
  \STATE $\hat{x} = x - \frac{1}{L} \nabla g (x)$ {\em Gradient Step}
  \STATE $z^0 \leftarrow 0$ 
	\FOR{$i=1$ to $l$} 
   		\STATE $z^i = \mathcal{A}_\text{in}(\hat{x},z^{i-1})$ {\em Proximal Step}
 	\ENDFOR
 	\STATE $\hat{x} = z^l$
 \IF{$f(x) - f(\hat{x}) < \text{tol} f(x)$}
   \STATE $l \leftarrow l+1$ {\em Increase proximal iterations}
 \ENDIF
  \STATE $x = \mathcal{A}_\text{out}(x,\hat{x})$ {\em Basic or accelerated update}
 \UNTIL{$\text{STOP}$ is met}
\end{algorithmic}
\end{algorithm}

Beyond the simplicity of the algorithm (no parameter except for the tolerance, no need to set a global accuracy in advance),
 \emph{SIP} leverages the observation that a constant number of inner iterations $l$ only allows to reach some underlying accuracy.
As long as this accuracy has not been reached, it is not necessary to put more efforts into estimating the proximity operator.
The rough idea is that far from the minimum of a convex function, moving along a rough estimation of the steepest direction will
be very likely to have the function decrease fast enough, hence the low precision required for the proximal point estimation.
On the other hand, when close to the minimum, a much higher precision is required, hence the need for using more inner iterations.  
This point of view meets the one developed in \cite{boyles11} in the context of stochastic optimization, where the authors
suggest to use increasing batch sizes (along the optimization procedure) for the stochastic estimation of the gradient of 
functional to minimize, in order to achieve computational efficiency.

\section{Numerical Simulations}
\label{proxtrad:expe}

The objective of this section is to empirically investigate the behaviour of proximal-gradient methods when the proximity operator
 is estimated via a fixed number of iterations. We also assess the performance of the proposed SIP algorithm.
Our expectation is that a strategy with just one internal iteration will be computationally optimal
 only up to a certain accuracy, after which using two internal iterations will be more efficient and so on.
We consider an image deblurring problem with total variation regularization and a semi-supervised
 learning problem using two sublinear methods for computing the proximity operator.

\subsection{TV-regularization for image deblurring}

The problem of denoising or deblurring an images is often tackled via Total Variation regularization \cite{rudin1992nonlinear,chambolle2004algorithm,beck2009tv}.
The total variation regularizer allows one to preserve sharp edges and is defined as 
\begin{equation*}
g(x) = \lambda \sum_{i,j=1}^N \|(\nabla x)_{i,j}\|_2 \;
\end{equation*}
where $\lambda > 0$ is a regularization parameter and $\nabla$ is the discrete gradient operator \cite{chambolle2004algorithm}.
We use the smooth quadratic data fit term $f(x) = \|Ax - y\|_2^2$,
where $A$ is a linear blurring operator and $y$ is the image to be deblurred.
This leads to the following problem:
$$\min_{x} \|Ax - y\|_2^2 + \lambda \sum_{i,j=1}^N \|(\nabla x)_{i,j}\|_2 .$$
Our experimental setup follows the one in \cite{villa2011aifobos}, where it was used for an asymptotic analysis.
We start with the famous Lena test image, scaled to $256 \times 256$ pixels. A $9 \times 9$ Gaussian filter with standard deviation $4$ is used to blur the image. Normal noise with zero mean and standard deviation $10^{-3}$ is also added. The regularization parameter $\lambda$ was set to $10^{-4}$. 
We run the basic proximal-gradient method up to a total computational cost of $C = 10^6$ (where we set $\Cin = \Cout = 1$) and the accelerated method up to a cost of $5 \times 10^4$. We computed the proximity operator using the algorithm of~\cite{beck2009tv}, which is a basic proximal-gradient method applied to the dual of the proximity operator problem. We used a fixed number of iterations and compared with the convergent strategy proposed in~\cite{Schmidt11} and the SIP algorithm with tolerance $10^{-8}$. As a reference for the optimal value of the objective function, we used the minimum value achieved by any method (i.e. the SIP algorithm in all cases) and reported the results in Fig.~\ref{fig:tv}.

\begin{figure}
\begin{tabular}{cc}
\includegraphics[width=0.47\textwidth]{tv_deblur_obj_vs_comp_cost_fixedK_no_acc.pdf} &
\includegraphics[width=0.47\textwidth]{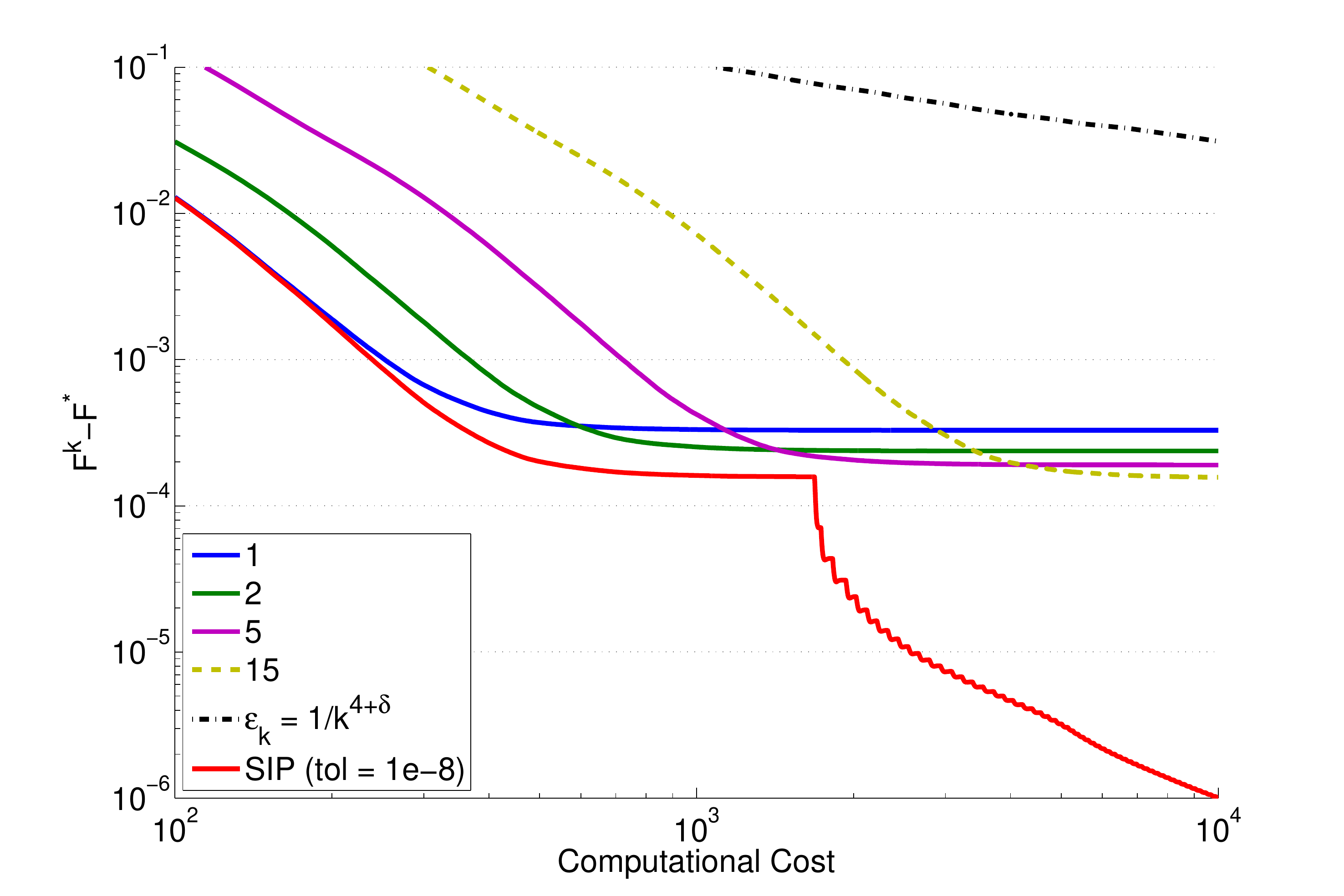}
\end{tabular}
\caption{Deblurring with Total Variation - Basic method (left) and Accelerated method (right)}
\label{fig:tv}
\end{figure}

As the figures display a similar behaviour for the different problems we ran our simulations on, we defer the analysis of
the results to \ref{subsec:why}.

\subsection{Graph prediction}
The second simulation is on the graph prediction setting of~\cite{mark09}.
It consists in a sequential prediction of boolean labels on the vertices of a graph, the learner's goal
 being the minimization of the number of mistakes.
More specifically, we consider a 1-seminorm on the space of graph labellings, which corresponds to the
 minimization of the following problem (composite $\ell_1$ norm)
$$
\min_{x} \|Ax - y\|^2 + \lambda \|Bx\|_1,
$$
where $A$ is a linear operator that selects only the vertices for which we have labels $y$, $B$ is the edge map of the graph and $\lambda > 0$ is a regularization parameter (set to $10^{-4}$).
We constructed a synthetic graph of $d = 100$ vertices, with two clusters of equal
size. The edges in each cluster were selected from a uniform draw
with probability $\frac{1}{2}$ and we explicitly connected $d/25$ 
pairs of vertices between the clusters. The labelled data $y$ were
the cluster labels ($+1$ or $-1$) of $s=10$ randomly drawn vertices.
We compute the proximity operator of $\lambda \|Bx\|_1$ via the method proposed in~\cite{combettes2010dualization},
 which essentially is a basic proximal method on the dual of the proximity operator problem.
 We follow the same experimental protocol as in the total variation problem and report the results in Fig.~\ref{fig:graph_acc}.

\begin{figure}
\begin{tabular}{cc}
\includegraphics[width=0.47\textwidth]{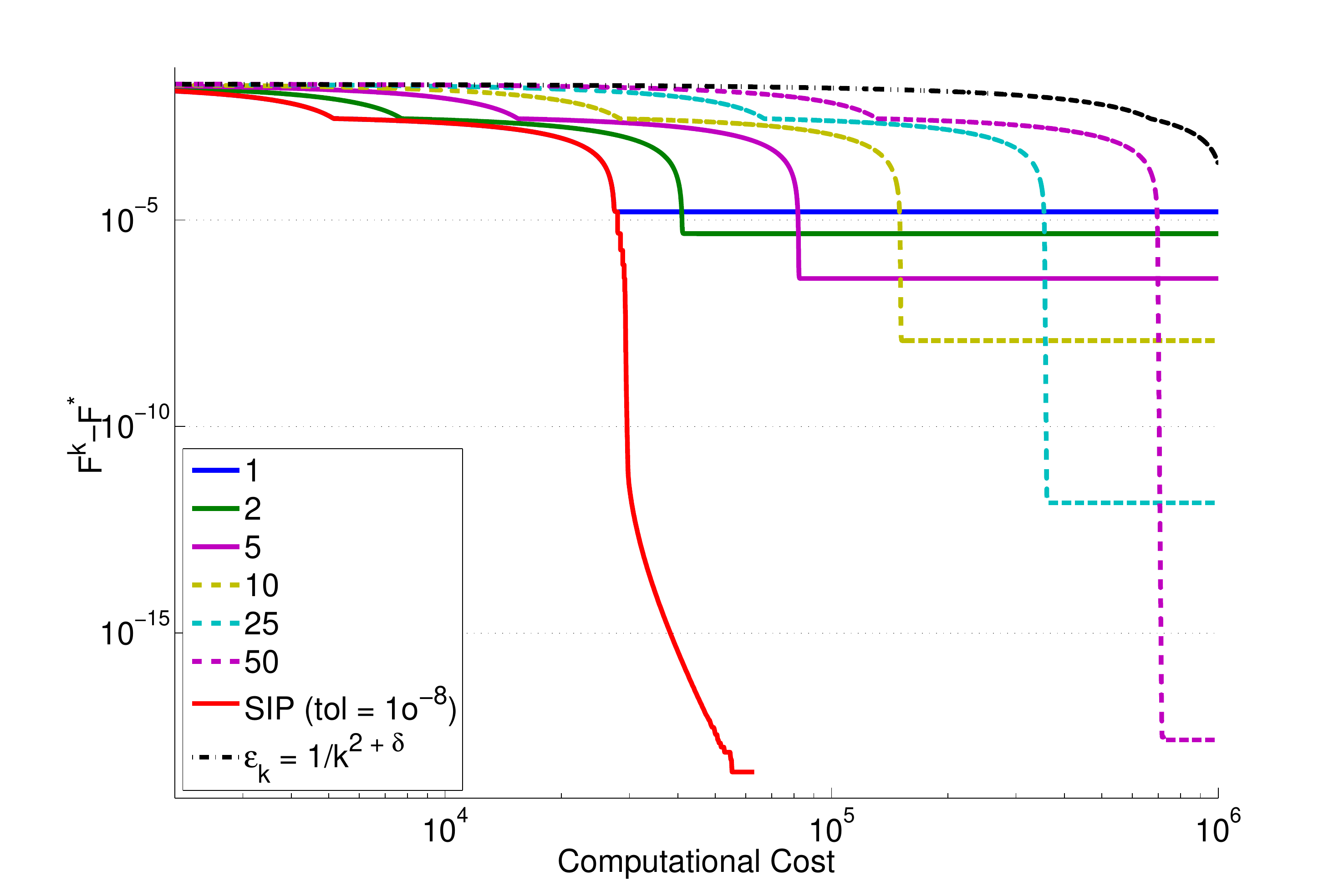} &
\includegraphics[width=0.47\textwidth]{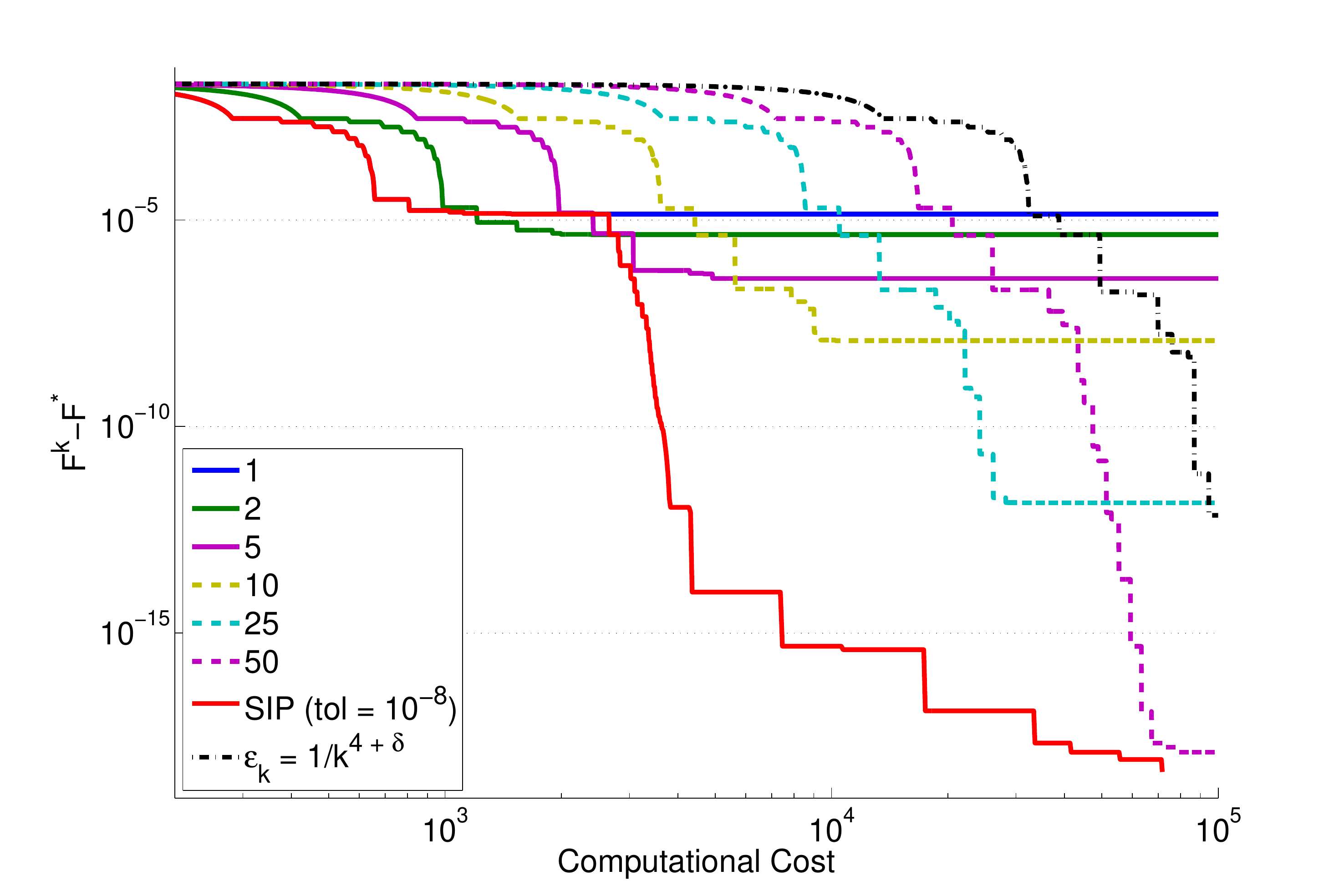}
\end{tabular}
\caption{Graph learning - Basic method (left) and Accelerated method (right)}
\label{fig:graph_acc}
\end{figure}



\subsection{Why the ``computationally optimal'' strategies are good but not that optimal}
\label{subsec:why}
On all the displayed results (Fig. \ref{fig:tv} and \ref{fig:graph_acc}),
 and as the theory predicted, we can see that for almost any given accuracy $\rho$ (i.e. $F^k - F^*$ on the figures),
 there exists some constant value for $l_i$ that yields a strategy that is potentially
 orders of magnitude more efficient than the strategy that ensures the fastest global convergence rate.
On any of the figures, comparing the curves obtained with 1 and 2 inner iterations, one may notice that the former
 first increases the precision faster than the latter. Meanwhile, the former eventually converges to a higher plateau than the latter.
This observation remains as the number of constant iterations increases.
This highlights the fact that smaller constant values of $l_i$ lead to faster algorithms at the cost of a worse global precision.
On the other hand, the \emph{SIP} strategy seems to almost always be the fastest strategy to reach any desired precision.
That makes it the most computationally efficient strategy as the figures show.
This may look surprising as the constant $l_i$'s strategies are supposed to be optimal for a specific precision and obviously are not. 

In fact, there is no contradiction with the theory: keeping $l_i$ constant leads to the optimal strategies for minimizing a bound on the real error,
 which can be significantly different than directly minimizing the error.

This remark raises crucial issues.
If the bound we use for the error was a perfect description of the real error, 
the strategies with constant $l_i$ would be the best also in practice.
Intuitively, the tighter the bounds, the closest our theoretical optimal strategy will be from the actual optimal one.
This intuition is corroborated by our numerical experiments. In our parametrization of $\epsilon_i$, in a first
 approximation, we decided to consider constant $A_i$ (see equation (\ref{epsilonbound})).
When not using warm restarts between two consecutive outer iterations,
 our model of $\epsilon_i$ does describe the actual behaviour much more accurately and our theoretical optimal strategy seems much closer to the real optimal one.
To take warm starts into account, one would need to consider decreasing sequences of $A_i$'s.
Doing so, one can notice that in the first 3 scenarios, the optimal strategies would not consist in using constant number of inner iterations any longer,
 but only constant $\epsilon_i$'s, hence maintaining the same gap between optimal rates and computationally optimal strategies.

These ideas urge for a finer understanding on how optimization algorithms behave in practice.
Our claim is that one pivotal key to design practically efficient algorithms is to have new tools
 such as warm-start analysis and, perhaps more importantly, convergence bounds that are tighter for
 specific problems (i.e. ``specific-case'' analysis rather than the usual ``worst-case'' ones).

\section{Conclusion and future work}
\label{proxtrad:ccl}

We analysed the computational cost of proximal-gradient methods when the proximity operator is computed numerically. 
Building upon the results in \cite{Schmidt11}, we proved that the optimization strategies, using a constant number
 of inner iterations, can have very
 significant impacts on computational efficiency, at the cost of obtaining only a suboptimal solution. 
Our numerical experiments showed that these strategies do exist in practice, albeit it might be difficult to access them. 
Coincidentally, those theoretical strategies meet those of actual implementations and widely-used packages
 and help us understanding both their efficiency and limitations.
We also proposed a novel optimization strategy, the SIP algorithm, that can bring large computational savings in
 practice and whose theoretical analysis needs to be further developed in future studies. 
Throughout the paper, we highlighted the fact that finite-time analysis, such as ours, urges for a better understanding of (even standard) optimization procedures. 
There is a need for sharper and problem-dependent error bounds, as well as a better theoretical analysis of warm-restart, for instance.

Finally, although we focused on inexact proximal-gradient methods, the present work was inspired by the paper ``The Trade-offs of Large-Scale Learning''~\cite{Bottou07}.
Bottou and Bousquet studied the trade-offs between computational accuracy and statistical performance of machine learning methods and advocate for
 sacrificing the rate of convergence of optimization algorithms in favour of lighter computational costs.
At a higher-level, future work naturally includes finding other situations where such trade-offs appear and analyze them using a similar methodology.
 


\section{Appendix}
\label{proxtrad:app}
\subsection*{Proof of Proposition~\ref{bosi}}

In this scenario, we use non-accelerated outer iterations and sublinear inner iterations. 
Our optimisation problem thus reads:
\begin{align*}
 \min_k \min_{\{l_i\}_{i=1}^k} &\Cin \sum_{i=1}^k l_i + k \Cout\qquad
 \text{s.t. } \frac{L}{2k}\left(\|x_0 - x^*\| + 3 \sum_{i=1}^k \sqrt{\frac{2 A_i}{L l_i^\alpha}}\right)^2 \leq \rho.
\end{align*}

Let us first examine the constraint.
\begin{align*}
 &\frac{L}{2k}\left(\|x_0 - x^*\| + 3 \sum_{i=1}^k \sqrt{\frac{2 A_i}{L l_i^\alpha}}\right)^2 \leq \rho\\
\Leftrightarrow &\|x_0 - x^*\| + 3 \sum_{i=1}^k \sqrt{\frac{2 A_i}{L l_i^\alpha}} \leq \sqrt{\frac{2 k \rho}{L}}\\
\Leftrightarrow & \sum_{i=1}^k \sqrt{\frac{A_i}{l_i^\alpha}} \leq \frac{\sqrt{L}}{3\sqrt{2}}\left(\sqrt{\frac{2 k \rho}{L}} - \|x_0 - x^*\|\right)
\end{align*}
As a first remark, this constraint can be satisfied only if $$k \geq \frac{L}{2 \rho} \|x_0 - x^*\|^2.$$
However this always holds as this only implies that the number of outer iterations $k$ is larger than the amount we would need
 if the proximity operator could be computed exactly.


Let us recall that for any $i$, $A_i$ is such that $\epsilon_i \leq {A_i}/{l_i^\alpha}$.
For most iterative optimization methods, the tightest bounds (of this form) on the error are obtained
 for constants $A_i$ depending on: a) properties of the objective function at hand, b) the initialization.
To mention an example we have already introduced, for basic proximal methods, one can choose $$A_i = \frac{L}{2 l_i} \|(x_k)_0 - x_k^*\|,$$
where $(x_k)_0$ is the initialization for our inner-problem at outer-iteration $k$ and $x_k^*$ the optimal of this problem.
As the problem seems intractable in the most general case, we will first assume that $\forall i, A_i=A$.
This only implies that we don't introduce any prior knowledge on $\|(x_k)_0 - x_k^*\|$ at each iteration.
This is reasonable if, at each outer-iteration, we randomly initialize $(x_k)_0$ but may lead to looser bounds if we use
 wiser strategies such as warm starts.

With that new assumption on $A_i$, one can state that the former constraint will hold if and only if:
$$\sum_{i=1}^k \sqrt{\frac{1}{l_i^\alpha}} \leq \frac{\sqrt{L}}{3\sqrt{2 A}}\left(\sqrt{\frac{2 k \rho}{L}} - \|x_0 - x^*\|\right).$$
Let us first solve the problem of finding the $\{l_i\}_{i=1}^k$ for some fixed $k$.
We need to solve:
\begin{align*}
 \argmin_{\{l_i\}_{i=1}^k\in {\mathbb{N}^*}^k} &\Cin \sum_{i=1}^k l_i + k \Cout\qquad
 \text{s.t. } \sum_{i=1}^k \sqrt{\frac{1}{l_i^\alpha}} \leq \frac{\sqrt{L}}{3\sqrt{2 A}}\left(\sqrt{\frac{2 k \rho}{L}} - \|x_0 - x^*\|\right):=C_k,
\end{align*}
which is equivalent to solving:
\begin{align*}
 \argmin_{\{l_i\}_{i=1}^k\in {\mathbb{N}^*}^k} &\sum_{i=1}^k l_i\qquad
 \text{s.t. } \sum_{i=1}^k \sqrt{\frac{1}{l_i^\alpha}} \leq C_k.
\end{align*}
\begin{remark}
\label{const_unconst}
 $l_i \in {\mathbb{N}^*}^k \Rightarrow \sqrt{\frac{1}{l_i^\alpha}} \in ]0,1] \Rightarrow \sum_{i=1}^k \sqrt{\frac{1}{l_i^\alpha}} \leq k$.
So, if $C_k \geq k$, then the solution of the constrained problem is the solution of the unconstrained problem.
In that case, the trivial solution is $l_i = 1, \forall i$.
Moreover, if $l_i = 1, \forall i$ is the solution of the constrained problem, then $\sum_{i=1}^k \sqrt{\frac{1}{l_i^\alpha}} = k \leq C_k$.
As a consequence, the solution of the unconstrained problem is the solution of the constrained problem \emph{if and only if} $C_k \geq k$.
\end{remark}

We then have two cases to consider:
\paragraph*{Case 1: $C_k \geq k$}
As stated before, the optimum will be trivially reached for $l_i = 1, \forall i$.
Now, we need to find the optimal over $k$.
It consists in finding:
\begin{align*}
 \min_{k\in \mathbb{N}^*} &~k (\Cin + \Cout)\qquad
\text{s.t. }  C_k \geq k.
\end{align*}
Let us have a look at the constraint.
\begin{align*}
 C_k \geq k &\Leftrightarrow \frac{\sqrt{L}}{3\sqrt{2 A}}\left(\sqrt{\frac{2 k \rho}{L}} - \|x_0 - x^*\|\right) \geq k\\
&\Leftrightarrow \sqrt{\frac{2 k \rho}{L}} \geq \frac{3\sqrt{2 A}}{\sqrt{L}} k +  \|x_0 - x^*\|\\
&\Leftrightarrow \left( \sqrt{k} - \frac{\sqrt{\rho}}{6 \sqrt{A}}\right)^2 \leq \frac{\rho}{36 A} - \frac{\sqrt{L}\|x_0 - x^*\|}{3 \sqrt{2 A}}
\end{align*}
Then:
\begin{itemize}
 \item if $\frac{\rho}{36 A} <  \frac{\sqrt{L}\|x_0 - x^*\|}{3 \sqrt{2 A}}$ then there is no solution (i.e. $C_k < k, \forall k$).
 \item if $\frac{\rho}{36 A} \geq  \frac{\sqrt{L}\|x_0 - x^*\|}{3 \sqrt{2 A}}$ then, the constraint holds for 
 $$k \in \bigg[\Big(\frac{\sqrt{\rho}}{6 \sqrt{A}} - \sqrt{\frac{\rho}{36 A} - \frac{\sqrt{L}\|x_0 - x^*\|}{3 \sqrt{2 A}}}\Big)^2 ,
 \Big(\frac{\sqrt{\rho}}{6 \sqrt{A}} + \sqrt{\frac{\rho}{36 A} - \frac{\sqrt{L}\|x_0 - x^*\|}{3 \sqrt{2 A}}}\Big)^2\bigg]$$.
 The optimum will then be achieved for the smallest integer (if exists) larger than
 $\Big(\frac{\sqrt{\rho}}{6 \sqrt{A}} - \sqrt{\frac{\rho}{36 A} - \frac{\sqrt{L}\|x_0 - x^*\|}{3 \sqrt{2 A}}}\Big)^2$
 and smaller than 
  $\Big(\frac{\sqrt{\rho}}{6 \sqrt{A}} + \sqrt{\frac{\rho}{36 A} - \frac{\sqrt{L}\|x_0 - x^*\|}{3 \sqrt{2 A}}}\Big)^2$.
\end{itemize}

\paragraph*{Case 2: $C_k \leq k$}
As remark~\ref{const_unconst} shows, the solution of the constrained problem is different from the unconstrained one.
The solution of this \emph{integer} optimization problem is hard to compute. 
In a first step, we may relax the problem and solve it as if $\{l_i\}_{i=1}^k$ were continuous variables taking values
into $[1,+\infty[^k$.
Because both our objective function and the constraints are continuous with respect to $\{l_i\}_{i=1}^k$,
 the optimal (over $\{l_i\}_{i=1}^k$) of our problem will precisely lie on the constraint.
Our problem now is:
\begin{align*}
 \argmin_{\{l_i\}_{i=1}^k\in [1,+\infty[^k} &\sum_{i=1}^k l_i\qquad
 \text{s.t. } \sum_{i=1}^k l_i^{-\frac{\alpha}{2}} = C_k.
\end{align*}

For any $i \in [1,k]$, let $n_i :=  l_i^{-\frac{\alpha}{2}}$.
Our problem becomes:
\begin{align*}
 \argmin_{\{n_i\}_{i=1}^k\in ]0,1]^k} &\sum_{i=1}^k n_i^{-\frac{2}{\alpha}}\qquad
 \text{s.t. } \sum_{i=1}^k n_i = C_k.
\end{align*}

Introducing the Lagrange multiplier $\lambda \in \mathbb{R}$, the Lagrangian of this problem writes:
\begin{align*}
 L(\{n_i\}_{i=1}^k,\lambda) := \sum_{i=1}^k n_i^{-\frac{2}{\alpha}} + \lambda \left(\sum_{i=1}^k n_i - C_k\right).
\end{align*}
 
And it follows that, $\forall i \in [1,k]$, when the optimum $\{n_i^*\}_{i=1}^k$ is reached:
\begin{align*}
 \frac{\partial L}{\partial n_i} = 0 \Leftrightarrow n_i^* = \left(\frac{\alpha \lambda}{2}\right)^{\frac{1}{-\frac{2}{\alpha}-1}}
\end{align*}

And now, plugging into our constraint:
\begin{align*}
 \sum_{i=1}^k n_i^* = C_k \Rightarrow \lambda = \frac{2}{\alpha} \left(\frac{C_k}{k}\right)^{-\frac{2}{\alpha}-1}.
\end{align*}

Hence, for any $i \in [1,k]$, $n_i^* = \frac{C_k}{k}$.

As $C_k \leq k$, it is clear that $\forall p, n_p^* \in ]0,1]$ and we have, $\forall i, l_i^* = \left(\frac{C_k}{k}\right)^{-\frac{2}{\alpha}}$.

We can now plug the optimal $l_i^*$ in our first problem and we now need to find the optimal $k^*$ such that:
\begin{align*}
 k^* &= \argmin_{k \in \mathbb{N}^*} \Cglob(k,\{l_i^*\}_{i=1}^k)\\
 &= \argmin_{k \in \mathbb{N}^*} \Cin \sum_{i=1}^k l_i^* + k \Cout\\
 &= \argmin_{k \in \mathbb{N}^*} \Cin \sum_{i=1}^k \left(\frac{C_k}{k}\right)^{-\frac{2}{\alpha}} + k \Cout\\
&= \argmin_{k \in \mathbb{N}^*} k \left( \Cin \Big(\frac{C_k}{k}\Big)^{-\frac{2}{\alpha}} + \Cout\right).
\end{align*}

Once again, we can relax this integer optimization problem into a continuous one, assuming $k \in \mathbb{R}^+$.
It directly follows that the solution of that relaxed problem is reached when the derivative
 (w.r.t. $k$) of $\Cglob(k,\{l_i^*\}_{i=1}^k)$ equals $0$.
The derivative can be easily computed:
\begin{align*}
 \frac{\partial \Cglob(k,\{l_i^*\}_{i=1}^k)}{\partial k} = \Cin \left(\Big( \frac{2}{\alpha} + 1 \Big) k^{\frac{2}{\alpha}} C_k^{-\frac{2}{\alpha}}
 -\frac{2}{\alpha} C_k' C_k^{-\frac{2}{\alpha}-1} k^{\frac{2}{\alpha}+1}\right) + \Cout,
\end{align*}
where $C_k'$ is the derivative of $C_k$ w.r.t. $k$:
$$C_k' = \frac{\sqrt{\rho}}{3 \sqrt{A}} k^{-\frac{1}{2}}.$$
However, giving an analytic form of that zero is difficult.
But using any numeric solver, it is very easy to find a very good approximation of $k^*$.
As described in Remark \ref{rem:kint}, this allows us to exactly retrieve the exact integer minimizer.


\subsection*{Proof of Proposition~\ref{boli}}

In this scenario, we use non-accelerated outer iterations and linear inner iterations. 
Our optimisation problem thus reads:
\begin{align*}
 \min_k \min_{\{l_i\}_{i=1}^k} &\Cin \sum_{i=1}^k l_i + k \Cout\qquad
 \text{s.t. } \frac{L}{2k}\left(\|x_0 - x^*\| + 3 \sum_{i=1}^k \sqrt{\frac{2 A_i(1-\gamma)^{l_i}}{L}}\right)^2 \leq \rho.
\end{align*}

We consider $A_i=A$. The error in the $i$th inner iteration reads:
\begin{align}
\label{epsilonboundlin}
\epsilon_i = A (1-\gamma)^{l_i}.
\end{align}

Hence the corresponding bound on the error:
\begin{align}
 \rho_k \leq \frac{L}{2k}\left(\|x_0 - x^*\| + 3 \sum_{i=1}^k \sqrt{\frac{2 A (1-\gamma)^{l_i}}{L}}\right)^2.
\end{align}

Problem in $\{l_i\}$ boils down to:
\begin{align*}
 \argmin_{\{l_i\}_{i=1}^k\in {\mathbb{N}^*}^k} &\sum_{i=1}^k l_i\qquad
 \text{s.t. } \sum_{i=1}^k (1-\gamma)^{\frac{l_i}{2}} \leq C_k,
\end{align*}
still with $C_k = \frac{\sqrt{L}}{3\sqrt{2 A}}\left(\sqrt{\frac{2 k \rho}{L}} - \|x_0 - x^*\|\right)$.

\paragraph*{Case 1: $C_k \geq k \sqrt{1-\gamma}$}
identical except for the threshold, which will also impact the interval for $k^*$.

\paragraph*{Case 2: $C_k \leq k \sqrt{1-\gamma}$}
For any $i \in [1,k]$, let $n_i :=  (1-\gamma)^{\frac{l_i}{2}}$.
Our problem becomes:
\begin{align*}
 \argmin_{\{n_i\}_{i=1}^k\in ]0,\sqrt{1-\gamma}]^k} &-\sum_{i=1}^k \ln n_i\qquad
 \text{s.t. } \sum_{i=1}^k n_i = C_k.
\end{align*}

Writing again the Lagrangian of this new problem, we obtain the same result: for any $i \in [1,k]$, $$n_i^* = \frac{C_k}{k}.$$
This leads to $$l_i^* = \frac{2 \ln \left(\frac{C_k}{k}\right)}{\ln(1-\gamma)}.$$

Following the same reasoning, we now plug this analytic solution of the first optimization problem into the second one.
This leads to:
\begin{align*}
 k^* = \argmin_{k \in \mathbb{N}^*}  k \left(\frac{2 \Cin}{\ln(1-\gamma)} \ln \Big(\frac{C_k}{k}\Big) + \Cout\right)\\
\end{align*}

This time, the derivative  of the continuous relaxation writes:
\begin{align*}
 \frac{\partial \Cglob(k,\{l_i^*\}_{i=1}^k)}{\partial k} = \frac{2 \Cin}{\ln(1-\gamma)} \left(\ln\frac{C_k}{k} + \frac{k C_k'}{C_k} - 1\right) + \Cout,
\end{align*}
where $C_k'$ is the derivative of $C_k$ w.r.t. $k$:
$$C_k' = \frac{\sqrt{\rho}}{3 \sqrt{A}} k^{-\frac{1}{2}}.$$

The optimum $k^*$ of our problem is the (unique) zero of that derivative.

\subsection*{Proof of Proposition~\ref{aosi}}

In this scenario, we use accelerated outer iterations and sublinear inner iterations. 
Our optimisation problem thus reads:
\begin{align*}
 \min_k \min_{\{l_i\}_{i=1}^k} &\Cin \sum_{i=1}^k l_i + k \Cout\qquad
 \text{s.t. } \frac{L}{(k+1)^2}\left(\|x_0 - x^*\| + 3 \sum_{i=1}^k i\sqrt{\frac{2 A_i}{Ll_i^\alpha}}\right)^2 \leq \rho.
\end{align*}

We consider $A_i=A$. The error in the $i$th inner iteration reads:
\begin{align}
\epsilon_i = \frac{A}{l_i^{\alpha}}.
\end{align}

Similarly, for the accelerated case, we have:
\begin{align}
 \rho_k \leq \frac{2 L}{(k+1)^2}\left(\|x_0 - x^*\| + 3 \sum_{i=1}^k i \sqrt{\frac{2 A_i}{L l_i^\alpha}}\right)^2.
\end{align}

Those problems can naturally be extended with the use of accelerated schemes and we get this ``error-oriented'' problem:
\begin{align*}
 \min_{k,\{l_i\}_{i=1}^k} &\Cin \sum_{i=1}^k l_i + k \Cout\qquad
 \text{s.t. } \frac{2 L}{(k+1)^2}\left(\|x_0 - x^*\| + 3 \sum_{i=1}^k i \sqrt{\frac{2 A_i}{L l_i^\alpha}}\right)^2 \leq \rho.
\end{align*}

We will follow the same reasoning as for the non-accelerated case.
We will consider this optimization problem:
\begin{align*}
 \min_k \min_{\{l_i\}_{i=1}^k} &\Cin \sum_{i=1}^k l_i + k \Cout\qquad
 \text{s.t. } \frac{2 L}{(k+1)^2}\left(\|x_0 - x^*\| + 3 \sum_{i=1}^k i \sqrt{\frac{2 A_i}{L l_i^\alpha}}\right)^2 \leq \rho.
\end{align*}
Let us first have a look at the constraint.
\begin{align*}
 &\frac{2 L}{(k+1)^2}\left(\|x_0 - x^*\| + 3 \sum_{i=1}^k i \sqrt{\frac{2 A_i}{L l_i^\alpha}}\right)^2 \leq \rho\\
\Leftrightarrow &\|x_0 - x^*\| + 3 \sum_{i=1}^k i \sqrt{\frac{2 A_i}{L l_i^\alpha}} \leq \sqrt{\frac{\rho}{2 L}} (k+1)\\
\Leftrightarrow & \sum_{i=1}^k i \sqrt{\frac{A_i}{l_i^\alpha}} \leq \frac{\sqrt{L}}{3\sqrt{2}}\left(\sqrt{\frac{\rho}{2 L}} (k+1) - \|x_0 - x^*\|\right)
\end{align*}
As in the former case, this can only hold if $(k+1) \geq \sqrt{\frac{2 L}{\rho}}\|x_0 - x^*\|$ which is trivial.

We will now assume again that $A_i = A$ for any $i$.
As earlier, we first solve the following problem in $ \{l_i\}_{i=1}^k$:
\begin{align*}
 \argmin_{\{l_i\}_{i=1}^k\in {\mathbb{N}^*}^k} &\sum_{i=1}^k l_i\qquad
 \text{s.t. } \sum_{i=1}^k i \sqrt{\frac{1}{l_i^\alpha}} \leq \frac{\sqrt{L}}{3\sqrt{2 A}}\left(\sqrt{\frac{\rho}{ 2L}} (k+1) - \|x_0 - x^*\|\right)=:D_k.
\end{align*}
\begin{remark}
\label{const_unconst_acc}
 $l_i \in {\mathbb{N}^*}^k \Rightarrow \sqrt{\frac{1}{l_i^\alpha}} \in ]0,1] \Rightarrow \sum_{i=1}^k i \sqrt{\frac{1}{l_i^\alpha}} \leq \frac{k(k+1)}{2}$.
So, if $D_k \geq \frac{k(k+1)}{2}$, then the solution of the constrained problem is the solution of the unconstrained problem.
In that case, the trivial solution is $l_i = 1, \forall i$.
Moreover, if $l_i = 1, \forall i$ is the solution of the constrained problem, then $\sum_{i=1}^k i \sqrt{\frac{1}{l_i^\alpha}} = \frac{k(k+1)}{2} \leq D_k$.
As a consequence, the solution of the unconstrained problem is the solution of the constrained problem \emph{if and only if} $D_k \geq \frac{k(k+1)}{2}$.
\end{remark}

\paragraph*{Case 1: $D_k \geq \frac{k(k+1)}{2}$}
As stated before, the optimum will be trivially reached for $l_i = 1, \forall i$.
Now, we need to find the optimal over $k$.
It consists in finding:
\begin{align*}
 \min_{k\in \mathbb{N}^*} &~k (\Cin + \Cout)\qquad
\text{s.t. } D_k \geq \frac{k(k+1)}{2}.
\end{align*}
Let us have a look at this constraint.
\begin{align*}
 D_k \geq \frac{k(k+1)}{2} &\Leftrightarrow \frac{\sqrt{L}}{3\sqrt{2 A}}\left(\sqrt{\frac{\rho}{ 2L}} (k+1) - \|x_0 - x^*\|\right) \geq \frac{k(k+1)}{2}\\
&\Leftrightarrow k^2 + k \left(1- \frac{\sqrt{\rho}}{3 \sqrt{A}}\right) \leq \frac{\sqrt{\rho}}{3 \sqrt{A}} - \frac{\sqrt{2L}}{3\sqrt{A}} \|x_0 - x^*\|\\
&\Leftrightarrow \left(k + \frac{1}{2} \Big(1- \frac{\sqrt{\rho}}{3 \sqrt{A}}\Big)\right)^2 \leq - \frac{\sqrt{2L}}{3\sqrt{A}} \|x_0 - x^*\|
 + \frac{1}{4} \Big(1 + \frac{\sqrt{\rho}}{3 \sqrt{A}}\Big)^2=:K.
\end{align*}
Then:
\begin{itemize}
 \item if $K < 0$ then there is no solution (i.e. $D_k < \frac{k(k+1)}{2}, \forall k$).
 \item if $K \geq 0$ then, the constraint holds for 
 $k \in \left[\frac{1}{2} \Big(\frac{\sqrt{\rho}}{3 \sqrt{A}} - 1\Big) - \sqrt{K}, \frac{1}{2} \Big(\frac{\sqrt{\rho}}{3 \sqrt{A}} - 1\Big) + \sqrt{K}\right]$.
 The optimum will then be achieved for the smallest integer (if exists) larger than
 $\frac{1}{2} \Big(\frac{\sqrt{\rho}}{3 \sqrt{A}} - 1\Big) - \sqrt{K}$
 and smaller than 
  $\frac{1}{2} \Big(\frac{\sqrt{\rho}}{3 \sqrt{A}} - 1\Big) + \sqrt{K}$.
\end{itemize}

\paragraph*{Case 2: $D_k \leq \frac{k(k+1)}{2}$}
Once again, we fall in the same scenario as in the non-accelerated case.
The solution of our problem is different from the unconstrained one and we can relax our discrete optimization problem to a continuous one.
The optimal then precisely lies again on the constraint.
We now have:
\begin{align*}
 \min_{\{l_i\}_{i=1}^k} &\sum_{i=1}^k l_i \qquad
 \text{s.t. } \sum_{i=1}^k i \sqrt{\frac{1}{l_i^\alpha}} = D_k.
\end{align*}

For any $i \in [1,k]$, let $n_i :=  i l_i^{-\frac{\alpha}{2}}$.
Our problem becomes:
\begin{align*}
 \min_{\{n_i\}_{i=1}^k} &\sum_{i=1}^k \left(\frac{n_i}{i}\right)^{-\frac{2}{\alpha}} \qquad
 \text{s.t. } \sum_{i=1}^k n_i = D_k.
\end{align*}

The Lagrangian writes:
\begin{align*}
 L(\{n_i\}_{i=1}^k,\lambda) := \sum_{i=1}^k \left(\frac{n_i}{i}\right)^{-\frac{2}{\alpha}} + \lambda \left(\sum_{i=1}^k n_i - D_k\right).
\end{align*}
 
And it follows that, $\forall i \in [1,k]$, when the optimum $\{n_i^*\}_{i=1}^k$ is reached:
\begin{align*}
 \frac{\partial L}{\partial n_i} = 0 \Leftrightarrow n_i^* = i \left(\frac{\alpha \lambda}{2}\right)^{\frac{1}{-\frac{2}{\alpha}-1}}
\end{align*}

And now, plugging into our constraint:
\begin{align*}
 \sum_{i=1}^k n_i^* = D_k \Rightarrow \lambda = \frac{2}{\alpha} \left(\frac{2 D_k}{k(k+1)}\right)^{-\frac{2}{\alpha}-1}.
\end{align*}

Hence, for any $i \in [1,k]$, $n_i^* = \frac{2 D_k}{k(k+1)} i$, giving the corresponding
 $l_i^* = \left(\frac{2 D_k}{k(k+1)}\right)^{-\frac{2}{\alpha}}$.

We can now plug the optimal $l_i^*$ in our first problem and we now need to find the optimal $k^*$ such that:
\begin{align*}
 k^* &= \argmin_{k \in \mathbb{N}^*} \Cglob(k,\{l_i^*\}_{i=1}^k).\\
 &= \argmin_{k \in \mathbb{N}^*} k \left(\Cin \Big(\frac{2 D_k}{k(k+1)}\Big)^{-\frac{2}{\alpha}} + \Cout\right).
\end{align*}

Once again, we can relax this integer optimization problem into a continuous one, assuming $k \in \mathbb{R}^+$.
It directly follows that the solution of that relaxed problem is reached when the derivative
 (w.r.t. $k$) of $\Cglob(k,\{l_i^*\}_{i=1}^k)$ equals $0$.

\subsection*{Proof of Proposition~\ref{aoli}}

In this scenario, we use accelerated outer iterations and linear inner iterations. 
Our optimisation problem thus reads:
\begin{align*}
 \min_k \min_{\{l_i\}_{i=1}^k} &\Cin \sum_{i=1}^k l_i + k \Cout\qquad
 \text{s.t. } \frac{L}{(k+1)^2}\left(\|x_0 - x^*\| + 3 \sum_{i=1}^k i\sqrt{\frac{2 A_i(1-\gamma)^{l_i}}{L}}\right)^2 \leq \rho.
\end{align*}

We consider $A_i=A$. The error in the $i$th inner iteration reads:
\begin{align}
\epsilon_i = A{(1-\gamma)}^l_i.
\end{align}

The problem in $\{l_i\}$ boils down to:
\begin{align}
 \argmin_{\{l_i\}_{i=1}^k\in {\mathbb{N}^*}^k} &\sum_{i=1}^k l_i\qquad
 \text{s.t. } \sum_{i=1}^k i (1-\gamma)^{\frac{l_i}{2}} \leq D_k,
 \label{eq:pb_acc_lin_integer}
\end{align}
with $D_k = \frac{\sqrt{L}}{3\sqrt{2 A}}\left(\sqrt{\frac{\rho}{ 2L}} (k+1) - \|x_0 - x^*\|\right)$.

\paragraph*{Case 1: $D_k \geq \frac{k(k+1)}{2} \sqrt{1-\gamma}$}
identical except for the threshold, which will also impact the interval for $k^*$.

\paragraph*{Case 2: $D_k \leq \frac{k(k+1)}{2} \sqrt{1-\gamma}$\\}
Relaxing Problem~\eqref{eq:pb_acc_lin_integer} to real numbers, we want to solve:
\begin{align}
 \argmin_{\{l_i\}_{i=1}^k\in \mathbb{R^+}^k}  \sum_{i=1}^k l_i\qquad\text{s.t. } 
 & \sum_{i=1}^k i (1-\gamma)^{\frac{l_i}{2}} - D_k \leq 0\\
 & 1-l_i\leq 0, \forall i.
 \label{eq:pb_acc_lin_real}
\end{align}

According to the KKT conditions, there exist $\{\mu_i\},\ i=1,..,k$ and $\lambda$, such that the optimum $\{l_i^*\}$ verify:
\begin{align}
 \text{(stationarity) } &   1+\lambda i  (1-\gamma)^{\frac{l_i^*}{2}} \ln(\sqrt{1-\gamma}) -\mu_i = 0, \quad \forall i=1,..,k \label{eq:KKT_stat}\\
 \text{(primal feasibility) } & \sum_{i=1}^k i (1-\gamma)^{\frac{l_i^*}{2}} - D_k \leq 0,
 \label{eq:KKT_prim}\\
 &1- l_i^* \leq 0, \quad \forall i=1,..,k, \label{eq:KKT_prim_i}\\
 \text{(dual feasibility) } &\lambda \geq 0, \label{eq:KKT_dual}\\
 &\mu_i \geq 0, \quad \forall i=1,..,k, \label{eq:KKT_dual_i}\\
 \text{(complementary slackness) } 
 &\lambda (\sum_{i=1}^k i(1-\gamma)^{\frac{l_i^*}{2}} - D_k )=0 , \label{eq:KKT_comp}\\
 &\mu_i (1-l_i^*) = 0, \quad \forall i=1,..,k. \label{eq:KKT_comp_i}
\end{align}

Eq.~\eqref{eq:KKT_dual} yields two cases: $\lambda=0$ or  $\lambda>0$.

\paragraph*{$\lambda = 0$ }
Then Eq.~\eqref{eq:KKT_stat} yields $\mu_i=1, \forall i$ thus Eq.\eqref{eq:KKT_comp_i} implies $l_i^*=1$. All the KKT conditions are thus fulfilled if Eq.\eqref{eq:KKT_prim} is, i.e. if
$$ D_k \geq \frac{k(k+1)}{2} \sqrt{1-\gamma}.$$

We work here in the case where $D_k \leq \frac{k(k+1)}{2} \sqrt{1-\gamma}$ thus this solution is valid if and only if  $D_k = \frac{k(k+1)}{2} \sqrt{1-\gamma}$.

\paragraph*{$\lambda > 0$ }

Again, Eq.~\eqref{eq:KKT_dual} yields two cases: $\mu_i=0$ or  $\mu_i>0$.

\paragraph*{Subcase 1: $\mu_i>0$\\}

Then by Eq.~\eqref{eq:KKT_comp_i}, we have $l_i^*=1$ and by~\eqref{eq:KKT_stat} $\mu_i = 1 + \lambda i \sqrt{1-\gamma} \ln(\sqrt{1-\gamma})$. Then $\mu_i > 0$ implies:
$$ i <{\frac{1}{\lambda\sqrt{1-\gamma} \ln(\sqrt{\tfrac{1}{1-\gamma}})} }.$$

\paragraph*{Subcase 2: $\mu_i=0$\\}

Then by Eq.~\eqref{eq:KKT_stat} we have $ 1 + \lambda i (1-\gamma)^{\frac{l_i^*}{2}} \ln(\sqrt{1-\gamma}) = 0$, i.e:

$$ l_i^*  = \frac{ \ln\left(i\lambda \ln(\sqrt{\tfrac{1}{1-\gamma}}) \right)}{\ln(\sqrt{\tfrac{1}{1-\gamma}})}.$$ 

Since Eq.~\eqref{eq:KKT_prim_i} enforces $l_i^*\leq 1$, we have:
$$i \geq{\frac{1}{\lambda\sqrt{1-\gamma} \ln(\sqrt{\tfrac{1}{1-\gamma}})} } .$$

\paragraph*{Conclusion:} For $\lambda>0$, Eq.~\eqref{eq:KKT_stat}, \eqref{eq:KKT_prim_i}, \eqref{eq:KKT_dual}, \eqref{eq:KKT_dual_i} and \eqref{eq:KKT_comp_i} are fullfilled all at once if we set:
\begin{equation}
  \begin{array}{lll}
  \text{For } i=1..\lceil\frac{1}{\lambda\sqrt{1-\gamma} \ln(\sqrt{\tfrac{1}{1-\gamma}})} \rceil -1: & l_i=1 & \mu_i =  1 + \lambda i \sqrt{1-\gamma} \ln(\sqrt{1-\gamma}) \\
 \text{For } i=\lceil\frac{1}{\lambda\sqrt{1-\gamma} \ln(\sqrt{\tfrac{1}{1-\gamma}})} \rceil ,..,k: & l_i= \frac{ \ln\left(i\lambda \ln(\sqrt{\tfrac{1}{1-\gamma}}) \right)}{\ln(\sqrt{\tfrac{1}{1-\gamma}})}& \mu_i = 0.
  \end{array} 
\end{equation}

With these values set for $\mu_i$ and $l_i^*$, let us now find the value of $\lambda$.

\paragraph*{Computing $\lambda$\\}

We need to fulfill Eq.~\eqref{eq:KKT_prim} and \eqref{eq:KKT_comp}.

Let us define $ M(\lambda) = \lceil\frac{1}{\lambda\sqrt{1-\gamma} \ln(\sqrt{\tfrac{1}{1-\gamma}})} \rceil$. \\
Note that for $\lambda>\frac{1}{(k+1)\lambda\sqrt{1-\gamma} \ln(\sqrt{\tfrac{1}{1-\gamma}})}$, we have: $0< M(\lambda)\leq k+1$, and:
\begin{itemize}
  \item $ M(\lambda) = 1 \Leftrightarrow \lambda \geq\frac{1}{\lambda\sqrt{1-\gamma} \ln(\sqrt{\tfrac{1}{1-\gamma}})} $
  \item $ M(\lambda) = n \Leftrightarrow \frac{1}{n\lambda\sqrt{1-\gamma} \ln(\sqrt{\tfrac{1}{1-\gamma}})} <\lambda <\frac{1}{(n-1)\lambda\sqrt{1-\gamma} \ln(\sqrt{\tfrac{1}{1-\gamma}})} $ for $n=2,..,k+1$.
\end{itemize}

Eq.~\eqref{eq:KKT_prim} and \eqref{eq:KKT_comp} are true if and only if
\begin{align*}
D_k &= \sum_{i=1}^k i  (1-\gamma)^{\frac{l_i^*}{2}} \\
D_k &= \frac{M(\lambda)(M(\lambda)-1)}{2} \sqrt{1-\gamma} + \frac{k-M(\lambda)+1}{\lambda\ln(\sqrt{\tfrac{1}{1-\gamma}})}.
\end{align*}

We define $F:\mathbb{R}^{+*}\to\mathbb{R} $ by $ F(\lambda) = \frac{M(\lambda)(M(\lambda)-1)}{2} \sqrt{1-\gamma} + \frac{k-M(\lambda)+1}{\lambda\ln(\sqrt{\tfrac{1}{1-\gamma}})}$.

Examining $F$ on each interval where $M$ is constant, it is easy to see that $F$ is continuous and non-increasing. Moreover $F$ decreases strictly on $[\frac{1}{k\lambda\sqrt{1-\gamma} \ln(\sqrt{\tfrac{1}{1-\gamma}})} ,\infty)$, $\lim_{\lambda\to\infty} F = 0$ and $F$ reaches its highest value $\max{F} = \frac{k(k+1)}{2} \sqrt{1-\gamma}$ on  $[\frac{1}{(k+1)\lambda\sqrt{1-\gamma} \ln(\sqrt{\tfrac{1}{1-\gamma}})} ,\frac{1}{k\lambda\sqrt{1-\gamma} \ln(\sqrt{\tfrac{1}{1-\gamma}})}] $.

We thus have for all $D_k$ such that $0<D_k < \frac{k(k+1)}{2} \sqrt{1-\gamma}$, there exists a unique $\lambda$ such that $F(\lambda) = D_k$ and thus all KKT conditions are fullfilled.

To find this value of $\lambda$ as a function of $D_k$, we first find $M(\lambda)$ from $D_k$. Notice that 
$$ F\left(\frac{1}{n\lambda\sqrt{1-\gamma} \ln(\sqrt{\tfrac{1}{1-\gamma}})} \right)= \frac{n(2k+1-n)}{2} \sqrt{1-\gamma}.$$
 
As $D_k< \frac{k(k+1)}{2} \sqrt{1-\gamma}$, there exists a unique integer $n$ in $1,..,k$ such that
\begin{equation}
 \frac{(n-1)(2k+2-n)}{2} \sqrt{1-\gamma} \leq D_k <  \frac{n(2k+1-n)}{2} \sqrt{1-\gamma}.
 \label{eq:nk}
\end{equation}
Then $M(\lambda)=n$ and the KKT conditions are all fulfilled for:
$$\lambda = \frac{k+1-n}{(D_k-\frac{n(n-1)}{2}\sqrt{1-\gamma}) \ln(\sqrt{\tfrac{1}{1-\gamma}})}.$$
In particular:
\begin{equation}
  \begin{array}{ll}
  \text{For } i=1,..,n-1: & l_i=1. \\
 \text{For } i=n,..,k: & l_i= \frac{ \ln\left(\frac{k+1-n}{D_k-\frac{n(n-1)}{2}\sqrt{1-\gamma}}\right) }{\ln(\sqrt{\frac{1}{1-\gamma}})}
  \end{array} 
\end{equation}

\paragraph*{Back to the global problem}

We now seek to find the the value $k^*$ that minimizes the global problem. Outisde of the interval defined in \emph{Case 1}, the global cost is defined by the following.
Let us define $n(k)$ as the integer verifying Eq.~\eqref{eq:nk}. Then

\begin{align*}
 C_{glob}(k) = kC_{out} &+ C_{in} (n(k)-1) + \frac{C_{in}(k-n(k)+1)}{ \ln(\sqrt{\tfrac{1}{1-\gamma}})} \ln\left(\frac{k+1-n(k)}{D_k-\frac{n(k)(n(k)-1)}{2}\sqrt{1-\gamma}}\right)\\
&+ \frac{C_{in}}{ \ln(\sqrt{\tfrac{1}{1-\gamma}})}\ln\left(\frac{k!}{n(k)!}\right).
\end{align*}

\bibliographystyle{plain}
\bibliography{proxTradeoff}

\begin{thebibliography}{10}

\bibitem{anthoine12}
S.~Anthoine, J.-F. Aujol, C.~Mélot, and Y.~Boursier.
\newblock Some proximal methods for cbct and pet tomography. inverse problems
  in imaging.
\newblock Accepted to Inverse Problems and Imaging, 2012.

\bibitem{bach2011convex}
F.~Bach, R.~Jenatton, J.~Mairal, and G.~Obozinski.
\newblock {\em Optimization for Machine Learning}, chapter Convex optimization
  with sparsity-inducing norms, pages 19--54.
\newblock MIT Press, 2011.

\bibitem{bach11b}
F.~Bach and E.~Moulines.
\newblock Non-asymptotic analysis of stochastic approximation algorithms for
  machine learning.
\newblock In {\em Advances in Neural Information Processing Systems (NIPS)},
  2011.

\bibitem{baldassarre2012general}
L.~Baldassarre, J.~Morales, A.~Argyriou, and M.~Pontil.
\newblock A general framework for structured sparsity via proximal
  optimization.
\newblock In {\em AISTATS}, 2012.

\bibitem{beck2009tv}
A.~Beck and M.~Teboulle.
\newblock Fast gradient-based algorithms for constrained total variation image
  denoising and deblurring problems.
\newblock {\em IEEE Trans. on Im. Proc.}, 18(11):2419--2434, 2009.

\bibitem{Beck09}
A.~Beck and M.~Teboulle.
\newblock A fast iterative shrinkage-thresholding algorithm for linear inverse
  problems.
\newblock {\em SIAM Journal on Imaging Sciences}, 2(1):183--202, 2009.

\bibitem{Bottou07}
L.~Bottou and O.~Bousquet.
\newblock The trade-offs of large scale learning.
\newblock In {\em Adv. in Neural Information Processing Systems (NIPS)}, 2007.

\bibitem{boyles11}
L.~Boyles, A.~Korattikara, D.~Ramanan, and M.~Welling.
\newblock Statistical tests for optimization efficiency.
\newblock In {\em Advances in Neural Information Processing Systems (NIPS)},
  2011.

\bibitem{cai2010singular}
J.F. Cai, E.J. Cand{\`e}s, and Z.~Shen.
\newblock A singular value thresholding algorithm for matrix completion.
\newblock {\em SIAM Journal on Optimization}, 20:1956, 2010.

\bibitem{chambolle2004algorithm}
A.~Chambolle.
\newblock An algorithm for total variation minimization and applications.
\newblock {\em J. Math. Imaging Vis.}, 20:89--97, 2004.

\bibitem{chambolle2011first}
A.~Chambolle and T.~Pock.
\newblock A first-order primal-dual algorithm for convex problems with
  applications to imaging.
\newblock {\em Journal of Mathematical Imaging and Vision}, 40:120--145, 2011.

\bibitem{chaux09}
C.~Chaux, J.-C. Pesquet, and N.~Pustelnik.
\newblock Nested iterative algorithms for convex constrained image recovery
  problems.
\newblock {\em SIAM Journal on Imaging Sciences}, 2:730--762, 2009.

\bibitem{chen2011smoothing}
X.~Chen, Q.~Lin, S.~Kim, J.G. Carbonell, and E.~P. Xing.
\newblock Smoothing proximal gradient method for general structured sparse
  learning.
\newblock In {\em UAI'11}, pages 105--114, 2011.

\bibitem{chen2009accelerated}
X.~Chen, W.~Pan, J.T. Kwok, and J.G. Carbonell.
\newblock Accelerated gradient method for multi-task sparse learning problem.
\newblock In {\em Ninth IEEE Intern. Conf. on Data Mining (ICDM '09).}, pages
  746 --751, dec. 2009.

\bibitem{combettes2010dualization}
P.L. Combettes, {D}. D{\~u}ng, and B.C. V{\~u}.
\newblock Dualization of signal recovery problems.
\newblock {\em Set-Valued and Variational Analysis}, pages 1--32, 2010.

\bibitem{combettes2005signal}
P.L. Combettes and V.R. Wajs.
\newblock Signal recovery by proximal forward-backward splitting.
\newblock {\em Multiscale Modeling and Simulation}, 4(4):1168--1200, 2005.

\bibitem{fadili2011total}
J.M. Fadili and G.~Peyr\'e.
\newblock Total variation projection with first order schemes.
\newblock {\em Image Processing, IEEE Transactions on}, 20(3):657--669, 2011.

\bibitem{mark09}
M.~Herbster and G.~Lever.
\newblock {Predicting the labelling of a graph via minimum p-seminorm
  interpolation}.
\newblock In {\em Proc. of the 22nd Conference on Learning Theory}, 2009.

\bibitem{jenatton2011proximal}
R.~Jenatton, J.~Mairal, G.~Obozinski, and F.~Bach.
\newblock Proximal methods for hierarchical sparse coding.
\newblock {\em Journal of Machine Learning Reasearch}, 12:2297--2334, 2011.

\bibitem{KKT}
H.~W. Kuhn and A.~W. Tucker.
\newblock Nonlinear programming.
\newblock In {\em Proc. of the {S}econd {B}erkeley {S}ymposium on
  {M}athematical {S}tatistics and {P}robability, 1950}, pages 481--492.
  University of California Press, 1951.

\bibitem{lin2009fast}
Z.~Lin, A.~Ganesh, J.~Wright, L.~Wu, M.~Chen, and Y.~Ma.
\newblock Fast convex optimization algorithms for exact recovery of a corrupted
  low-rank matrix.
\newblock {\em Computational Advances in Multi-Sensor Adaptive Processing
  (CAMSAP)}, 2009.

\bibitem{loris2011generalization}
I.~Loris and C.~Verhoeven.
\newblock On a generalization of the iterative soft-thresholding algorithm for
  the case of non-separable penalty.
\newblock {\em Inverse Problems}, 27:125007, 2011.

\bibitem{mosci2010solving}
S.~Mosci, L.~Rosasco, M.~Santoro, A.~Verri, and S.~Villa.
\newblock Solving structured sparsity regularization with proximal methods.
\newblock In {\em Machine Learning and Knowledge Discovery in Databases},
  volume 6322 of {\em Lecture Notes in Computer Science}, pages 418--433.
  Springer, 2010.

\bibitem{nemirovsky1983}
A.S. Nemirovsky and D.B. Yudin.
\newblock {\em Problem complexity and method efficiency in optimization.}
\newblock Wiley-Interscience Series in Discrete Mathematics. Johon Wiley \&
  Sons, New York, 1983.

\bibitem{Nesterov07}
Y.~Nesterov.
\newblock Gradient methods for minimizing composite objective function.
\newblock Technical report, CORE Discussion Papers, 2007.

\bibitem{polyak92}
B.~T. Polyak and A.~B. Juditsky.
\newblock Acceleration of stochastic approximation by averaging.
\newblock {\em SIAM Journal on Control and Optimization}, 30:838--855, 1992.

\bibitem{rudin1992nonlinear}
L.I. Rudin, S.~Osher, and E.~Fatemi.
\newblock Nonlinear total variation based noise removal algorithms.
\newblock {\em Physica D: Nonlinear Phenomena}, 60(1-4):259--268, 1992.

\bibitem{Schmidt11}
M.~Schmidt, N.~Le~Roux, and F.~Bach.
\newblock Convergence rates of inexact proximal-gradient methods for convex
  optimization.
\newblock In {\em Adv. in Neural Information Processing Systems (NIPS)}, 2011.

\bibitem{Tseng08}
P.~Tseng.
\newblock On accelerated proximal gradient methods for convex-concave
  optimization.
\newblock Submitted to SIAM Journals on Optimization, 2008.

\bibitem{villa2011aifobos}
S.~Villa, S.~Salzo, L.~Baldassarre, and A.~Verri.
\newblock Accelerated and inexact forward-backward algorithms.
\newblock Technical report, Optimization Online, 2011.

\end{thebibliography}


\end{document}